\PassOptionsToPackage{table}{xcolor}
\documentclass[10pt,twocolumn,letterpaper]{article}

\def\ARXIVVERSION{}

\ifdefined\ARXIVVERSION
\usepackage[pagenumbers]{cvpr}
\else
\usepackage{cvpr}
\fi

\usepackage{times}  
\usepackage{amsmath}  
\usepackage{amssymb}  
\usepackage{amsthm}   
\usepackage[hyphens]{url}  
\usepackage{graphicx}

\usepackage{algorithm}
\usepackage{algpseudocode}

\usepackage{tabularx}

\usepackage{multirow}
\usepackage{booktabs}  
\usepackage[inkscapelatex=false]{svg}
\usepackage{arydshln}
\usepackage{enumitem}
\usepackage{natbib}  
\usepackage{newfloat}
\usepackage{listings}

\usepackage{soul}
\setuldepth{foobar}

\ifx\ARXIVVERSION\undefined
\usepackage{xr}
\fi

\renewcommand{\paragraph}[1]{\vspace{.5em}\noindent\textbf{#1}}
\renewcommand\labelenumi{\thesubsection.\arabic{enumi}}

\newcounter{checksubsection}
\newcounter{checkitem}[checksubsection]

\newcommand{\maybeResize}[2][1]{\resizebox{#1\linewidth}{!}{#2}}

\newcommand{\norm}[1]{\left\lVert#1\right\rVert}
\newcommand{\E}{\mathbb{E}}
\newcommand{\Var}{\text{Var}}
\newcommand{\proofstep}[1]{\noindent\textbf{#1}\ }

\DeclareMathOperator*{\argmax}{arg\,max}
\DeclareMathOperator*{\argmin}{arg\,min}

\newtheorem{theorem}{Theorem}

\newtheorem{proposition}[theorem]{Proposition}

\newtheorem{assumption}{Assumption}
\newtheorem{remark}{Remark}

\def\secref#1{§~\ref{#1}}
\def\figref#1{Fig.~\ref{#1}}
\def\tabref#1{Tab.~\ref{#1}}
\def\eqref#1{Eq.~\ref{#1}}
\def\appref#1{Appendix~\ref{#1}}
\def\propref#1{Proposition~\ref{#1}}

\newcommand{\revdel}[1]{{\color{red}#1}}

\renewcommand{\revdel}[1]{}


\definecolor{cvprblue}{rgb}{0.21,0.49,0.74}
\usepackage[pagebackref,breaklinks,colorlinks,allcolors=cvprblue]{hyperref}

\def\paperID{35096} 
\def\confName{CVPR}
\def\confYear{2026}

\title{RLFTSim: Realistic and Controllable Multi-Agent Traffic Simulation via Reinforcement Learning Fine-Tuning}

\author {
    Ehsan Ahmadi\textsuperscript{\rm 1,\rm 2} \quad
    Hunter Schofield\textsuperscript{\rm 2,\rm 3} \quad
    Behzad Khamidehi\textsuperscript{\rm 2} \quad
    Fazel Arasteh\textsuperscript{\rm 2} \\
    Jinjun Shan\textsuperscript{\rm 3} \quad
    Lili Mou\textsuperscript{\rm 1,4} \quad
    Dongfeng Bai\textsuperscript{\rm 2} \quad
    Kasra Rezaee\textsuperscript{\rm 2} \\
    \textsuperscript{\rm 1}University of Alberta \ 
    \textsuperscript{\rm 2}Huawei Technologies Canada \\
    \textsuperscript{\rm 3}York University \ 
    \textsuperscript{\rm 4}Canada CIFAR AI Chair, Amii \\
    {\small eahmadi@ualberta.ca \  \{firstname.lastname\}@huawei.com  \ \{hunterls, jjshan\}@yorku.ca \ doublepower.mou@gmail.com}
}

\begin{document}

\maketitle

\begin{abstract}

  Supervised open-loop training has been widely adopted for training traffic simulation models; however, it fails to capture the inherently dynamic, multi-agent interactions common in complex driving scenarios. We introduce RLFTSim, a reinforcement-learning-based fine-tuning framework that enhances scenario realism by aligning simulator rollouts with real-world data distributions and provides a method for distilling goal-conditioned controllability in scenario generation. We instantiate RLFTSim on top of a pre-trained simulation model, design a reward that balances fidelity and controllability, and perform comprehensive experiments on the Waymo Open Motion Dataset. Our results show improvements in realism, achieving state-of-the-art performance. Compared with other heuristic search-based fine-tuning methods, RLFTSim requires significantly fewer samples due to a proposed low-variance and dense reward signal, and it directly addresses the realism alignment issue by design. We also demonstrate the effectiveness of our approach for distilling traffic simulation controllability through goal conditioning.
  The project page is available at \url{https://ehsan-ami.github.io/rlftsim}.

\end{abstract} %
\section{Introduction}

Simulation plays a major role in autonomous driving by providing a controllable virtual environment for the test and development of autonomous vehicles (AVs).
This is particularly significant for rare accident scenarios that an AV should be capable of handling.
Based on a statistical estimate in \cite{KALRA2016182}, verifying the safety aptitude of an AV requires it to be capable of driving 2 million kilometers without fatal accidents to claim that it is as safe as humans.
Considering the costs of such verification, the use of simulation is justified to make the verification and development of AVs feasible.

In this work, we focus on the problem of microscopic multi-agent traffic simulation, as we elaborate in \secref{background}.
This problem has previously been addressed with rule-based simulators, which can provide various feedback signals from kinematics and physics simulations \cite{gulino2024waymax,alban2025smart}.
These simulators control the agents either via replaying logged trajectories or using simple rule-based models, such as the constant-speed actor model or the Intelligent Driver Model \cite{idm}. This, arguably, leads to a significant sim-to-real gap.
Even for the case of ground-truth data log replay, since the agents' behavior is not reactive, the simulation becomes unrealistic when it faces closed-loop deployment.

\begin{figure}[!tp]
    \centering
    \includegraphics[width=\linewidth]{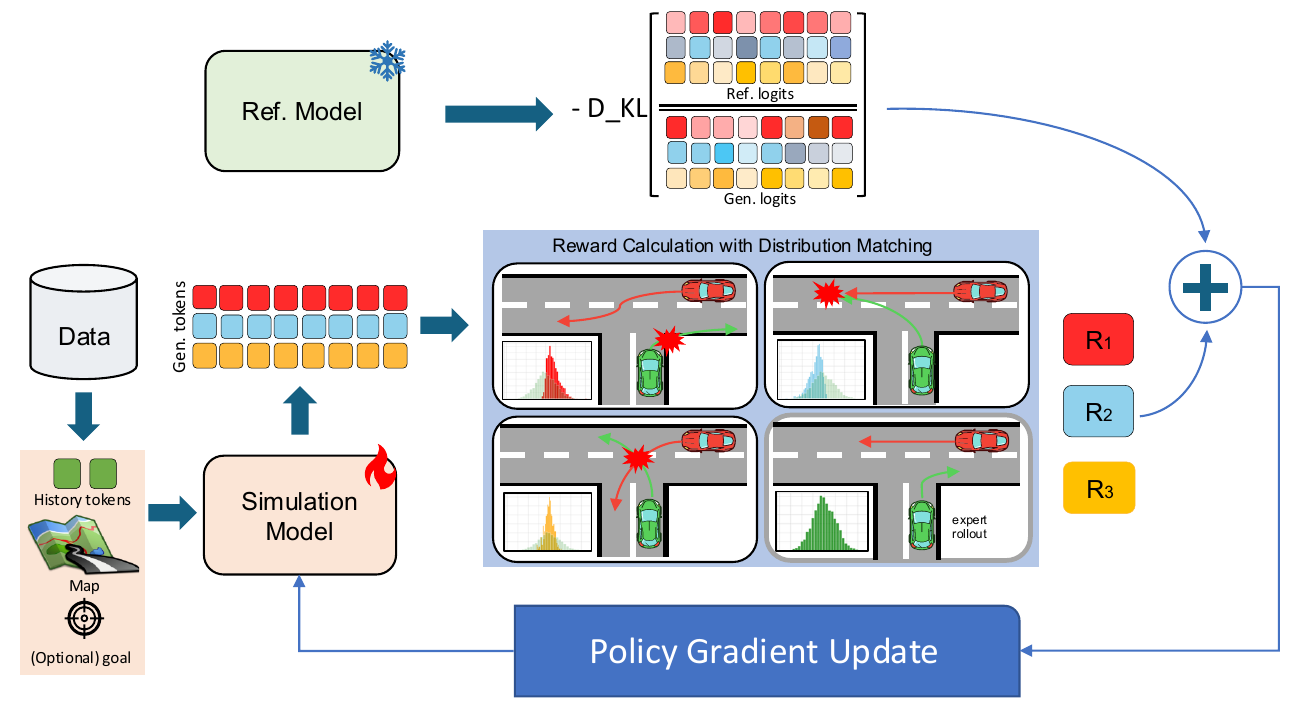}
    \caption{Post-training with RLFTSim. For a seed scenario from the dataset, multiple rollouts are generated. The main reward function is defined based on time-independent distribution matching of the simulated scenarios and the expert demonstration in several feature spaces. This closed-loop on-policy optimization enhances realism beyond what open-loop imitation learning achieves alone.
    }
    \label{fig:teaser}
\end{figure} 
Recently, learning-based simulation models have been introduced to bridge the realism gap and generate a reactive multi-agent traffic simulation rollout, which is an episode of the simulation for all of the agents in a traffic scenario \cite{Wu2024SMART, Zhou2024BehaviorGPT}. 
However, previous models are mostly trained in an open-loop setting with imitation learning objectives \cite{Wu2024SMART, Zhou2024BehaviorGPT}.
We argue that to make simulation models capable of overcoming distribution shifts caused by error accumulation in closed-loop simulation and avoiding the causal confusion problem \cite{causalConfusion, ahmadi2025curbattentioncausalattention}, the training (or post-training) of a simulation model needs to be done in a closed-loop environment.

An effective simulation model needs to be capable of generating realistic data that accurately represents daily driving scenarios to make its adoption worthwhile.
While still imitating human agents, it needs to be aligned with traffic rules and the physical constraints of the environment.
In the past, a combination of open-loop supervision via imitation learning and closed-loop training via reinforcement learning has been used for ego-centric motion planning \cite{Lu2023RobustILRL}. However, effective scaling of this approach for multi-agent simulation models remains an open area of research.

In this work, inspired by the successful application of reinforcement learning with verifiable rewards for new skill learning in foundation models \cite{shao2024deepseekmath}, 
we present a reinforcement learning-based fine-tuning approach for traffic simulation (RLFTSim) to address the physics and traffic-rules alignment problem for a pre-trained model with an open-loop imitation learning objective. 
We deal with the alignment problem with a closed-loop post-training approach (\figref{fig:teaser}).
Specifically, we focus on the Waymo Open Simulation Challenge (WOSAC) and use its definition of realism meta-metric (RMM) as the optimization objective function for RLFT.
However, RMM in its original form is a sparse population-based signal calculated per group of 32 rollouts, making it sample-inefficient for direct use as a reward in RL. 
To overcome this issue, we present realism Meta-metric Leave-One-Out (MLOO) as a low-variance and dense adaptation of the RMM for RL fine-tuning.
To the best of our knowledge, this is the first successful application of the RMM for RL fine-tuning in traffic simulation.
We present theoretical variance scaling analysis that justify our proposed MLOO as a reward signal.
Moreover, we discuss the issue of controllability in traffic simulation as a part of the model alignment problem, 
and with the aid of hindsight experience replay (HER) \cite{HER}, goal-conditioning, and RLFT, we present an approach to distill controllability.

Our contributions are: 
\textbf{(i)} We introduce realism meta-metric as the reward signal for realism alignment with closed-loop RL-based fine-tuning in multi-agent traffic simulation (\secref{method}).
\textbf{(ii)} We provide a theoretical variance analysis for the current WOSAC's realism meta-metric, and we propose MLOO to have a low-variance and dense reward signal (\secref{sec:mloo}).
\textbf{(iii)} We introduce the application of hindsight experience replay with RLFT for the behavior controllability distillation with goal-conditioning in the simulation process (\secref{sec:controllability}).
\textbf{(iv)} We conduct comprehensive experiments showcasing the enhanced realism with RLFT to achieve the state-of-the-art performance, and the distilled controllability skills (\secref{experiments}).

\section{Related Works}
\label{related}

\paragraph{Learning-based Traffic Simulation Approaches.}
Learning-based models have emerged as a cornerstone for realistic traffic simulation, moving beyond heuristic rule-based simulators that struggled with complex maneuvers and interactions \cite{idm}.
Early works like \cite{Suo2021TrafficSim} modeled multi-agent driving behavior with implicit latent variables and differentiable closed-loop training, significantly improving realism over rule-based baselines.
Imitation learning has been widely adopted for traffic simulation \cite{gump2024, kigras2025,lin2025revisitmixturemodelsmultiagent,  vbd2024}.
The Waymo Open Motion Dataset (WOMD) \cite{WOMD_MF}, nuPlan \cite{caesar2021nuplan}, and similar datasets of millions of real trajectories have enabled training transformer-based models that treat multi-agent traffic simulation as a \emph{next-token prediction} problem \cite{Wu2024SMART, hu2024solving, philion2024trajeglish,  zhang2025trajtok, Zhou2024BehaviorGPT}. 
SMART predicts the next motion token for each agent conditioned on vectorized state representations and map context \cite{Wu2024SMART}. By training on over a billion motion tokens from multiple datasets, SMART achieved state-of-the-art realism in the WOMD Sim Agents Challenge. 
In this work, we utilize the superior power of SMART in traffic simulation by choosing it as a reference model for alignment fine-tuning.

\paragraph{Model Alignment and Fine-Tuning.}
A key challenge in learned simulators is \emph{model alignment}: ensuring the distribution of simulated behaviors in closed-loop matches real-world driving data.
Pure behavior cloning suffers from the accumulation of errors during closed-loop deployment \cite{ross2011reduction}, where small mistakes compound and drive agents to unseen or unrealistic states.
In \cite{Zhang2024ClosedLoop}, \emph{Closest Among Top-K} (CAT-K) sampling is used to perturb the expert policy.
By generating multiple plausible next actions and selecting the one closest to the ground-truth trajectory for training, it enables closed-loop learning solely from offline demonstrations. 
However, the absence of explicit alignment objectives can limit adherence to traffic rules and physical constraints.

Reinforcement learning and other closed-loop fine-tuning techniques have been explored to directly optimize behaviors in traffic simulation \cite{Igl2022Symphony}. 
A core challenge is the lack of an explicit reward signal for human-like driving, making reward design difficult and often inadequate. 
To address this, recent work combines imitation learning with reinforcement learning. For instance, \cite{Lu2023RobustILRL} shows that fine-tuning an imitation-learned policy with safety rewards reduces collisions by over 38\% in rare scenarios. However, this approach, focused on single-agent planning, does not directly extend to multi-agent traffic simulation.

Recent studies have applied reinforcement learning from human preferences for traffic simulation alignment \cite{trafficrlhf2024,lanechangeRLHF2024}.
However, they rely on costly human feedback, which limits scalability.
A preliminary study in \cite{tian2025direct} shows that ranking a traffic simulation can take over 40 seconds, making large-scale human preference annotation impractical. 

To address the scalability limitation, \cite{tian2025direct} introduces Direct Preference Alignment from Occupancy Measure Feedback (DPA-OMF) for post-training realism alignment.
While this method avoids extra annotation, it is sample inefficient as it depends on oversampling traffic simulation rollouts and only uses a subset of them.
Moreover, it is categorized as an offline method, which comes with limitations compared to on-policy methods \cite{tang2024understandingperformancegaponline}.
In contrast, RLFTSim is an on-policy method that, by using MLOO as the reward signal, effectively utilizes all simulation rollouts for closed-loop realism alignment.

\paragraph{Controllability in Learned Simulators and Goal Conditioning.}
Beyond realism, an important frontier is \emph{controllability}: the ability to steer the simulator or specify conditions so that generated traffic scenarios meet particular requirements.
Early learned simulators mostly acted as black boxes, generating unconstrained traffic patterns.
Now, there is a growing demand for methods to control \emph{what} scenarios are produced.
One promising direction is goal-conditioning of agent policies.
For example, \cite{Zhong2023Diffusion} factors their policy into high-level goal generation and low-level trajectory execution.
By conditioning each agent on an explicit goal, the simulator can ensure diversity without sacrificing realism.
Similarly, \cite{Zhang2023TrafficBots} introduces configurable latent variables, such as destinations and driver “personality” traits, to modulate each agent’s behavior, enabling simulations that span from aggressive to cautious driving styles on demand.
Hierarchical imitation frameworks such as \cite{xu2023bitsICRA} also improve control by decoupling strategic decisions from micromanagement. Another approach to achieve controllability over agents is by integrating language models to enable promptable simulation as in \cite{tan2025prosim}.

Conditioning behavior on explicit goals in autonomous driving spans paradigms from target-driven prediction \cite{zhao2021tnt} to goal-based planning \cite{albrecht2021interpretable}.
While goals are often task-specific and not trivial to define, \cite{GCRL} surveys various formulations of goal-conditioned reinforcement learning, where goals are treated as desired properties or features of agent behavior.
In the context of autonomous driving, many motion prediction benchmarks \cite{WOMD_MF, Argoverse2} define a miss rate metric that evaluates how well a model can predict trajectories that result in final agent states close to ground truth.
In \cite{HRPPO}, agent rollouts are explicitly conditioned on the final coordinates of the ground-truth trajectories and optimized using a sparse reward aligned with these metrics.
Building upon these works, we aim to distill controllability within the SMART simulator while maintaining its ability to produce highly realistic rollouts.

 \vspace{0.5em}

\section{Background}
\label{background}

\paragraph{Multi-Agent Traffic Simulation.} We formulate multi-agent traffic behavior modeling as a Contextual Markov Decision Process defined by the tuple of $(S_t,A_t,S_{t+1},R_{t+1},C,G)$ with discrete time steps $t \in [1,T]$.
The state $S_t = \{S_{t'}^j \mid j \leq N_a, t' \leq t\}$ captures the finite-horizon history of up to $N_a$ agents, the joint action $A_t = \{A_t^j \in \mathcal{V} \mid j \leq N_a\}$ denotes their tokenized decisions over the discrete vocabulary $\mathcal{V}$, and the reward is represented as $R_{t+1} \in \mathbb{R}$.
The context includes vectorized static and dynamic map features $C = \{M \in \mathbb{R}^{N_{m} \times D_{m}}, L \in \mathbb{R}^{N_{tl} \times D_{tl}} \}$, where M is the static map information comprising up to $N_{m}$ downsampled vector points with the feature dimension $D_m$, and $L$ is the dynamic map information including the past temporal states of up to $N_{tl}$ traffic light states with the feature dimension $D_{tl}$.
The optional goal specification $G = \{\textbf{x}^j_g\}_{j \in \mathcal{G}}$ defines target objectives for a subset $\mathcal{G} \subseteq \{1,\ldots,N_a\}$ of agents, where $\textbf{x}^j_g$ represents the goal coordinate for agent $j$. 
The simulation is categorized as \textit{goal-free} if $\mathcal{G} = \emptyset$, and \textit{goal-conditioned} otherwise.

\paragraph{WOSAC Realism Meta-Metric (RMM).}
The Waymo Open Simulation Challenge evaluates simulator realism using a meta-metric \cite{Montali2023neurips_wosac} that compares feature distributions between simulated and real trajectories. 
Given $N=32$ simulated rollouts $\{\tau_i\}_{i=1}^N$ and a ground-truth trajectory $\tau^*$, RMM extracts kinematic features (e.g., speed, acceleration), interactive features (e.g., distance to other agents, time-to-collision), and map-based features (e.g., distance to road boundary), discretizing them into $K=20$ bins to compute:
\begin{equation}
\mathrm{RMM} = \sum_{d=1}^D w_d \left[\prod_{(a,t) \in V} \hat P_{d,a}(k^*_{d,a,t}) \right]^{\frac{1}{|V|}},
\label{eq:RMM}
\end{equation}
where $\hat P_{d,a}(k)$ is the empirical probability of agent $a$'s feature dimension $d \in \{1,\ldots,D\}$ at bin $k$ using the simulated rollouts, $k^*_{d,a,t}$ indicates the ground-truth bin, $w_d$ is the weight for feature dimension $d$, and $V$ is the set of (agent, time) pairs for evaluation. 
Higher RMM values indicate better alignment with real-world behavior. 
The complete RMM formulation is detailed in \appref{app:rmm_derivation}.

\paragraph{REINFORCE Leave-One-Out (RLOO).}
RLOO \cite{Kool2019Buy4R} provides variance-reduced policy gradients by using bias-free leave-one-out baselines. 
For $N$ rollouts from the same context, the advantage of rollout $i$ is computed as $R_i - \frac{1}{N-1}\sum_{j \neq i} R_j$, where rewards are computed independently for each rollout.
\section{Traffic Simulation Alignment with Reinforcement Learning}
\label{method}
In this section, we present RLFTSim, a reinforcement learning-based fine-tuning approach for traffic simulation alignment. 
RLFTSim operates in two distinct modes to address different simulation requirements: \textit{goal-free simulation} for pure realism alignment, and \textit{goal-conditioned simulation} for controllability while maintaining realism. 
Our post-training pipeline takes a pre-trained imitation learning model and fine-tunes it using reinforcement learning with carefully designed reward signals. 

The foundation of our approach lies in addressing a fundamental challenge: how to create an effective reward signal for RL-based simulation alignment. 
While RMM provides a comprehensive measure of traffic realism, its direct use as a reward signal is problematic due to sparsity and variance issues.
We introduce the Meta-metric Leave-One-Out (MLOO) method as a dense, low-variance replacement that enables sample-efficient training for both simulation modes.
Beyond pure realism, many practical applications require controllable simulations for targeted scenario testing and safety validation.
For such cases, we extend our approach with goal-conditioned simulation, combining MLOO with an auxiliary goal attainment reward to enable controllable agent behaviors without sacrificing realism.

We structure our method as follows: First, we address the challenges of using RMM as a reward signal and introduce MLOO as our core contribution (\secref{sec:mloo}).
Then, we present the architectural extensions and training strategies required for goal-conditioned simulation, including goal conditioning mechanisms and hindsight experience replay (\secref{sec:controllability}).

\subsection{Low-Variance and Dense Reward Signal with Meta-metric Leave-One-Out}
\label{sec:mloo}

A key challenge in using RL for traffic simulation alignment is designing an effective reward signal. 
A natural choice is to minimize the Average Distance Error (ADE) between the simulated trajectories and the ground-truth trajectory. Although this objective is used in the closed-loop RL fine-tuning \cite{peng2024_rlft_nvidia}, we argue that ADE is not a good choice as a reward signal for this setting. Once the agents diverge from the ground-truth trajectory due to the stochasticity of the simulation model, the ADE does not always provide good corrective supervision.
For instance, suppose that an agent is involved in a safety-critical situation in a simulation rollout; in this case, the most realistic next action may not be to return abruptly to a pre-recorded ground-truth trajectory.

RMM aims to capture the distribution of realistic driving behaviors using a comprehensive set of features \cite{Montali2023neurips_wosac}. It is less sensitive to the stochasticity of the simulation model as it relaxes the requirement of following the expert behavior at each time step (see \appref{app:rmm_derivation} for a detailed definition of RMM).
While RMM provides a comprehensive evaluation of traffic rollouts, its direct application as a reward signal is problematic. 
As a mapping from a set of 32 rollouts to a single scalar value, RMM is inherently sparse, which leads to sample inefficiency in RL training. 
Therefore, despite being the official metric for realism in the leading simulation benchmark, 
it has not been used as an optimization target for RL.
A naïve remedy for the sparsity problem is to compute RMM over smaller groups of rollouts (e.g., $N{=}4$ instead of $32$), yielding more reward values per batch; however, each RMM estimate is based on fewer samples, thus potentially destabilizing training due to higher variance.

\paragraph{Meta-metric Leave-One-Out (MLOO).} To achieve a better density-variance trade-off, we introduce a dense per-rollout reward signal, $\mathrm{RMM}_i^{\text{MLOO}}$, defined as:
\begin{equation}
    \mathrm{RMM}_i^{\text{MLOO}} = \frac{1}{N} \sum_{j=1}^N \mathrm{RMM}_{-j} - \mathrm{RMM}_{-i},
    \label{eq:MLOO}
\end{equation}
where $\mathrm{RMM}_{-j}$ is the realism meta-metric computed on the set of rollouts excluding the $j$-th rollout.
By construction, $\sum_{i=1}^N \mathrm{RMM}_i^{\mathrm{MLOO}} = 0$; the $\mathrm{RMM}_i^{\mathrm{MLOO}}$ signal measures each rollout's relative contribution rather than estimating the scalar RMM directly.

Our optimization uses REINFORCE with $\mathrm{RMM}_i^{\mathrm{MLOO}}$ as the per-rollout reward and Kullback--Leibler (KL) divergence regularization against the pre-trained reference model to maintain training stability (\appref{app:implementation_details}). Below, we formalize two theoretical properties of the MLOO gradient component: the resulting policy-gradient estimator is unbiased (\propref{prop:mloo_pg}), and MLOO achieves quadratic variance reduction in $N$ compared to per-rollout alternatives (\propref{prop:var-meta-metric-is}--\ref{prop:mloo-rloo-variance}). The leave-one-out construction targets $\mathbb{E}[\mathrm{RMM}(\tau_{1:N-1})]$, i.e., the expected RMM on $N{-}1$ rollouts rather than $N$; the difference is negligible for large $N$. Detailed proofs are provided in \appref{app:proofs}.

\begin{proposition}[Unbiased gradient estimation with MLOO]
\label{prop:mloo_pg}
Let $\tau_{1:N}=(\tau_1,\dots,\tau_N)$ be $N$ i.i.d.\ rollouts sampled from
the policy $\pi_\theta$. Applying REINFORCE with per-rollout reward $\mathrm{RMM}_i^{\mathrm{MLOO}}$ as defined in \eqref{eq:MLOO}, the policy-gradient estimator
\begin{equation}
g = \sum_{i=1}^{N}\nabla_\theta \log \pi_\theta(\tau_i)\,\mathrm{RMM}_i^{\mathrm{MLOO}}
\end{equation}
is an unbiased estimator of $\nabla_\theta \mathbb{E}\!\left[\mathrm{RMM}(\tau_{1:N-1})\right]$.
\end{proposition}

Having established unbiasedness, we now analyze the variance scaling of MLOO --- the property that determines its practical effectiveness for RL training.

\begin{assumption}
\label{ass:meta-metric-is}
Let $\mathcal{B} = \{B_1, \ldots, B_K\}$ be a partition of the feature space into $K$ bins. We assume:
{\renewcommand{\labelenumi}{\arabic{enumi}.}
\begin{enumerate}
    \item Feature observations across time steps are independent.
    \item The ground-truth trajectory $\tau^*$, via the empirical histogram of per-timestep feature values, induces a target distribution $p = (p_1, \ldots, p_K)$ where $p_k = \Pr[f^* \in B_k]$.
    \item The simulator produces samples from a proposal distribution $q = (q_1, \ldots, q_K)$ where $q_k = \Pr[f^{(sim)} \in B_k]$.
    \item The support condition holds: $\mathrm{supp}(p) \subseteq \mathrm{supp}(q)$.
\end{enumerate}}
\end{assumption}
\begin{remark}
\label{rem:independence-assumption}
Independence of time steps, Assumption 1.1, is a deliberate simplification assumption
considered in the original definition of Waymo Realism Meta-Metric, and it provides crucial flexibility, preventing the metric from reducing to an overly strict measure of ground-truth imitation at each time step \cite{Montali2023neurips_wosac}. Likewise, inter-agent correlations within a rollout (e.g., from interactive features) may further inflate the variance constants.
\end{remark}

To facilitate the variance analysis, we model the actual RMM computation where the simulator generates rollouts directly from distribution $q$, while $\alpha_k$ represents fixed empirical frequencies from the ground truth trajectory.

First, we establish baseline variance scaling for the realism meta-metric itself. This provides the foundation for understanding how leave-one-out methods improve upon direct meta-metric usage. Throughout, $\Var(\cdot)$ denotes variance over the simulated rollouts.

\begin{proposition}[Variance Scaling with Simulator Bias]
\label{prop:var-meta-metric-is}
Under Assumption~\ref{ass:meta-metric-is}, for $N$ rollouts of length $T$, the realism meta-metric~\eqref{eq:RMM} satisfies
\begin{equation}
\Var(\mathrm{RMM}) = O\!\left((\hat{N}_{\mathrm{eff}} \cdot T)^{-1}\right),
\end{equation}
where $\hat{N}_{\mathrm{eff}} = N / \hat{\kappa}$ is the effective sample size, $\hat{\kappa} = \max_{d}\,\kappa_d \geq 1$, and $\kappa_d = \sum_{k=1}^K \alpha_{k,d}^2 / q_{k,d}$ measures the mismatch between the simulator bin probabilities $q_{k,d}$ and the ground-truth bin frequencies $\alpha_{k,d}$ for feature dimension $d$.
\end{proposition}

\begin{remark}
\label{rem:importance-sampling}
When $q_{k,d} = \alpha_{k,d}$ for all $k,d$ (perfect simulator), $\kappa_d = 1$ for every $d$, so $\hat{N}_{\mathrm{eff}} = N$ and we recover $O((NT)^{-1})$ scaling. The effective sample size $\hat{N}_{\mathrm{eff}}$ decreases as the worst-case mismatch across feature dimensions increases, reflecting the variance penalty when the simulator is biased relative to the ground truth.
\end{remark}

Secondly, we analyze a key comparison: how do MLOO and RLOO perform relative to each other? The critical insight is that MLOO achieves quadratic variance reduction with more rollouts, while RLOO's variance remains constant.

\begin{proposition}[Variance of MLOO and RLOO Estimators]
\label{prop:mloo-rloo-variance}
Under Assumption~\ref{ass:meta-metric-is}, let $\{\mathrm{RMM}_{-i}\}_{i=1}^N$ be leave-one-out meta-metric estimates where $\mathrm{RMM}_{-i}$ is computed using $N-1$ rollouts excluding rollout $i$. Let $\mathrm{RMM}_i$ denote the single-rollout realism meta-metric computed from rollout $\tau_i$ alone, i.e., evaluating \eqref{eq:RMM} with $N{=}1$; by Proposition~\ref{prop:var-meta-metric-is}, $\Var(\mathrm{RMM}_i) = O(1/T)$. Define:
\begin{align}
\mathrm{RMM}_i^{\text{MLOO}} &= \frac{1}{N} \sum_{j=1}^N \mathrm{RMM}_{-j} - \mathrm{RMM}_{-i}, \\
\mathrm{RMM}_i^{\text{RLOO}} &= \mathrm{RMM}_i - \frac{1}{N-1} \sum_{j \neq i} \mathrm{RMM}_j.
\end{align}
Then the variances satisfy:
\begin{align}
\Var(\mathrm{RMM}_i^{\text{MLOO}}) &= O\left(\frac{1}{N^2 T}\right), \\ 
\Var(\mathrm{RMM}_i^{\text{RLOO}}) &= O\left(\frac{1}{T}\right). 
\end{align}
\end{proposition}

MLOO's quadratic variance reduction versus RLOO's constant variance relative to inverse number of rollouts provides substantial advantages for RL training stability, provided $\hat{N}_{\mathrm{eff}}$ scales reasonably with the number of rollouts $N$. 
Even with moderate simulator-ground truth mismatch, this scaling difference translates to more stable policy gradient estimates in practice, as empirically validated in \secref{exp:alt_reward}.
 
\subsection{Controllability with Goal Conditioning}
\label{sec:controllability}

State-of-the-art traffic simulation world models aim to capture realistic driving behaviors by maximizing the RMM.
While this provides a multi-modal traffic simulation model that can generate rollouts over likely driving distributions, the specific behaviors of individual agents remain stochastic, and targeted scenarios cannot be explored.

To distill controllability from an existing simulator, we perform goal-conditioned fine-tuning (GCFT), enabling the simulation model to attend to externally provided goals.
We explore multiple ways of extending a model's observation to include goal states to find a balance between simulation controllability and realism.

\paragraph{Goal Definition.} We identify two specific goal criteria representing hard and soft goals during training.
In the case of the hard target, agent $i$ is considered to have \textit{reached} its goal, $\textbf{x}^i_g = (x^i_g, y^i_g)$ if its final displacement is within $2.0$ meters of the goal coordinate.
For the soft target, the agent is considered to have \textit{passed} its goal if it passes within 2.0 meters of the goal coordinate at any point during the simulation rollout.
Since goal coordinates are often related to map information, we also introduce the notion of a goal polyline, $P^i_g$, defined as the polyline whose origin is closest to the agent's specified goal coordinate:

\begin{equation}
    P^i_g = \argmin_{\textbf{m} \in \mathcal{M}} \norm{\textbf{m} -  \textbf{x}^i_g},
\end{equation}
where $\mathcal{M}$ represents the set of all map polyline points. The goal polyline does not change the target goal coordinate, but provides additional contextual information for goal-conditioned control.

\paragraph{Goal Representation.} 
We explore two methods for including goal information in the agent's observation: 
i) \textit{agent token embedding concatenation (cat)}, which directly appends continuous goal coordinates to the agent's state, and 
ii) \textit{positional encoding indication (ind)}, which extends the relative positional encoding between the agent and 
road tokens with a binary indicator for the goal polyline. Detailed implementations are provided in \appref{app:goal_architectures}.

\paragraph{Hindsight Experience Replay.} Reaching specific goals can be a rare event, leading to sparse rewards. We take inspiration from Hindsight Experience Replay (HER) to augment data samples with new goals that the model is capable of reaching.
Algorithm \ref{alg:goal_sampling} describes this process for a given dataset. 
Specifically, a group of rollouts is generated from the same history prior, $S_{<t}$. In each step of the rollout, trajectory tokens are selected by performing temperature sampling over the top 32 trajectory tokens, providing stochastic rollouts. After generating a group of rollouts, the best rollout based on the RMM is selected, and the terminal states for each agent in this rollout are used as alternate goals, $\hat{X}_g = \{\hat{\textbf{x}}^j_g\}_{j \in \mathcal{G}}$, to augment the dataset. 
When training on these samples, we update the observed states of the rollouts such that the original goal, $\textbf{x}^i_g$, is replaced with the alternative goal, $\hat{\textbf{x}}^i_g$, and recompute the policy ratio,
$\hat{r}_{i,t}(\theta)=\frac{\pi_{\theta}(S^i_{\geq t} \mid S^i_{<t}, \hat{\textbf{x}}^i_g, C)}{\pi_{\theta}(S^i_{\geq t} \mid S^i_{<t}, \textbf{x}^i_g, C)}$, following the hindsight policy gradient framework \cite{rauber2018hindsight}.
This improves sample efficiency by providing intermediary goals in challenging scenarios where agents may not be able to reach the ground-truth goal, and can be used for either soft or hard goals.

\begin{table}[!tp]
\centering
\caption{Traffic simulation benchmarking results. Results are based on WOSAC leaderboard\protect\footnotemark{} evaluation on the private test split. We also present the results for our reference model (SMART). $(\uparrow)$ indicates that larger values are better. \textbf{Bold} and \underline{underline} indicate the best and second best values, respectively\protect\footnotemark{}. $\dagger$ indicates our re-trained version of the reference model.}
\label{tab:waymo_results}
\maybeResize[1]{
  \begin{tabular}{l>{\columncolor[gray]{0.8}}cccc}
    \toprule
    \rowcolor{white!50}
    Model
    & RMM$\uparrow$
    & Kinematic$\uparrow$
    & Interactive$\uparrow$ 
    & Map-based$\uparrow$  \\
    \midrule
    TrafficBotsV1.5 \cite{zhang2024trafficbotsv15} & 0.7167 & 0.4304 & 0.7114 & 0.8871 \\
    VBD \cite{vbd2024} & 0.7375 & 0.4169 & 0.7819 & 0.8636 \\
    MVTE \cite{mvte2023} & 0.7469 & 0.4503 & 0.7706 & 0.8859 \\
    Trajeglish \cite{philion2024trajeglish} & 0.7409 & 0.4166 & 0.7845 & 0.8703 \\
    KiGRAS \cite{kigras2025} & 0.7761 & 0.4691 & 0.8064 & 0.9126 \\
    DRoPE-Traj \cite{zhao2025drope} & 0.7786 & 0.4779 & 0.8065 & 0.9144 \\
    GUMP \cite{gump2024} & 0.7596 & 0.4780 & 0.7887 & 0.8832 \\
    BehaviorGPT \cite{Zhou2024BehaviorGPT} & 0.7637 & 0.4333 & 0.7997 & 0.9064 \\
    UniMM \cite{lin2025revisitmixturemodelsmultiagent} & 0.7839 & 0.4914 & 0.8089 & 0.9188 \\
    TrajTok \cite{zhang2025trajtok} & \underline{0.7861} & 0.4887 & \underline{0.8116} & \textbf{0.9231} \\
    \midrule
    SMART-tiny \cite{Wu2024SMART} & 0.7755 & 0.4759 & 0.8039 & 0.9102 \\
    SMART-tiny \cite{Wu2024SMART} (ref. model)$^{\dagger}$ & 0.7824 & 0.4854 & 0.8089 & 0.9180 \\
    SMART-tiny CAT-K \cite{Zhang2024ClosedLoop} & 0.7856 & \textbf{0.4931} & 0.8106 & 0.9205 \\
    \textbf{RLFTSim (ours)} & \textbf{0.7867} & \underline{0.4927} & \textbf{0.8129} & \underline{0.9210} \\
    \bottomrule
  \end{tabular}
}
\end{table}
\addtocounter{footnote}{-1}
\footnotetext{For consistent comparison, we present the results for all models using the meta-metric weights of v2025. Since the traffic light violation likelihood was not present in the previous years, we assume this metric is 1 for all models.}
\stepcounter{footnote}
\footnotetext{The bolding and underlining are done assuming that the private test set results have standard error values in a similar range as the validation set results in \tabref{tab:reward_ablation}.}
 
\paragraph{Realism Alignment.} The objective of performing GCFT is to distill simulation controllability while keeping the realistic simulation behavior. The per-rollout reward combines MLOO with a goal-reaching signal:
\begin{equation}
R_i^{\text{GCFT}} = (1-\lambda) \, \mathrm{RMM}_i^{\text{MLOO}} + \lambda \, R_i^{\text{goal}},
\label{eq:gcft_reward}
\end{equation}
where $R_i^{\text{goal}}$ is the mean binary goal-reaching reward over evaluated agents and $\lambda \in [0,1]$ balances realism and controllability (\appref{app:implementation_details}). This combined reward is used in the same REINFORCE framework as the goal-free case (\secref{sec:mloo}).

\begin{algorithm}[!tp]
    \caption{Stochastic Target Augmentation}\label{alg:goal_sampling}
    
    \textbf{Input:} 
    Scenario set $\mathcal{D}$, policy model $\pi_{\theta}$, group size $N_G$.

    \textbf{Output:} 
    Augmented Dataset $\mathcal{D}^*$.
    
    \begin{algorithmic}[1]
    \State Initialize augmented dataset $\mathcal{D}^*$
    \For{$(X_g, C, S_{<t}) \sim \mathcal{D}$ } \Comment{Get goals, context, history}
    \For{$i=1, \cdots, N_G$} \Comment{Sample $N_G$ rollouts} 
    \State $\tau_i \gets \pi_{\theta}(X_g, C, S_{<t})$ 
    \EndFor
    
    \State $\tau^* \gets \argmax_{\tau \in \{ \tau_j \}_{j=1}^{N_G}}
    \mathrm{RMM}(\tau)$. \Comment{eq. \ref{eq:RMM}}
    \State $\hat{X}_g \gets \tau^*_T$ \Comment{Get terminal state of best rollout}
    \State Store $(\hat{X}_g, C, S_{<t}, S_{\geq t})$ in $\mathcal{D}^*$

    \EndFor
    
    \State \Return $\mathcal{D}^*$.
    \end{algorithmic}
\end{algorithm}

\section{Experiments}
\label{experiments}

In this section, first we discuss the experimental setup (\secref{exp:setup}), then we present the results of our experiments to answer several research questions: 
\textbf{RQ1:} Is reinforcement learning-based post-training with the proposed meta-metric Leave-One-Out reward signal effective in enhancing the realism of simulated traffic rollouts? 
(\secref{exp:sim_realism}) 
\textbf{RQ2:} Can we use the self-supervised reconstruction error objective to enhance the realism of simulated scenarios? (\secref{exp:alt_reward}) 
\textbf{RQ3:} Is there an effective approach to condition rollouts on specific goals? (\secref{exp:goal_completion})
\textbf{RQ4:} Can we distill behavior controllability into a base simulation model with a proper reward design? (\secref{exp:goal_completion}, \appref{app:controllability})

\subsection{Experimental Setup}
\label{exp:setup}
We use the Waymo Open Motion Dataset  as it provides the standard benchmark for WOSAC evaluation and contains large-scale real-world driving trajectories necessary for multi-agent simulation.
The SMART-tiny base model is trained for 32 epochs on WOMD following \cite{Wu2024SMART}.
For the RLFT step, we use the learning rate 3e-6, target KL divergence of 0.01 nats, 4 rollouts, and a batch size of 8.
For the evaluation, the number of rollouts is set to 32 following the original configuration of the WOSAC's meta-metric.
A more detailed discussion of implementation details is given in  \appref{app:implementation_details}.

\subsection{Simulation Realism Benchmarking}
\label{exp:sim_realism}
The evaluation results for the simulation challenge are provided in \tabref{tab:waymo_results}. 
We use SMART-tiny as our base model \cite{Wu2024SMART}.
Then, after RL-based fine-tuning for 1 epoch, we get the SMART-RLFTSim model.
Compared to the base model, we can see that the realism meta-metric across all dimensions has improved (bottom row).
Our model achieves state-of-the-art performance in realism meta-metric (the primary metric) and the interactive metric, including over SMART-tiny CAT-K \cite{Zhang2024ClosedLoop}, another fine-tuning method that also uses SMART-tiny as its base model.
RLFTSim in \tabref{tab:waymo_results} uses $\mathrm{RMM}^{\mathrm{MLOO}}$ (\secref{sec:mloo}) as the reward signal for goal-free RL fine-tuning.

\begin{table}[!tp]
\centering
\caption{Ablation study on the reward function on the full validation set. Standard errors are shown in parentheses.}
\label{tab:reward_ablation}
\maybeResize[1]{
  \begin{tabular}{l>{\columncolor[gray]{0.8}}ccccc}
\toprule
\rowcolor{white!50}
Reward
& RMM$\uparrow$
& Kinematic$\uparrow$
& Interactive$\uparrow$
& Map-based$\uparrow$
& minADE$\downarrow$ \\
\midrule
  SMART-tiny \cite{Wu2024SMART} (ref. model) & 0.7804 \scriptsize{(3.2e-4)} & 0.4904 \scriptsize{(5.2e-4)} & 0.8032 \scriptsize{(4.1e-4)} & 0.9167 \scriptsize{(5.6e-4)} & \textbf{1.3016} \scriptsize{(4.2e-3)} \\
  \cdashline{1-6}
  $\text{minADE}^{\text{RLOO}}$ & 0.7801 \scriptsize{(3.3e-4)} & 0.4897 \scriptsize{(5.2e-4)} & 0.8032 \scriptsize{(4.1e-4)} & 0.9161 \scriptsize{(5.8e-4)} & 1.3202 \scriptsize{(4.5e-3)} \\
  $\text{RMM}^{\text{RLOO}}$ & 0.7821 \scriptsize{(3.3e-4)} & 0.4913 \scriptsize{(5.1e-4)} & 0.8065 \scriptsize{(4.2e-4)} & 0.9169 \scriptsize{(6.0e-4)} & 1.3229 \scriptsize{(4.4e-3)} \\
  $\text{RMM}^{\text{MLOO}}$ & \textbf{0.7830} \scriptsize{(3.3e-4)} & \textbf{0.4924} \scriptsize{(5.0e-4)} & \textbf{0.8070} \scriptsize{(4.1e-4)} & \textbf{0.9182} \scriptsize{(5.7e-4)} & 1.3150 \scriptsize{(4.4e-3)} \\
  Col.+Off.+ADE & 0.7803 \scriptsize{(3.3e-4)} & 0.4896 \scriptsize{(5.2e-4)} & 0.8039 \scriptsize{(4.1e-4)} & 0.9162 \scriptsize{(5.9e-4)} & 1.3313 \scriptsize{(4.5e-3)} \\
  Collision+Offroad & 0.7786 \scriptsize{(3.5e-4)} & 0.4891 \scriptsize{(5.2e-4)} & 0.8037 \scriptsize{(4.2e-4)} & 0.9117 \scriptsize{(6.4e-4)} & 1.3461 \scriptsize{(4.3e-3)} \\
\bottomrule
  \end{tabular}
}
\end{table}

\subsection{Alternative Reward Signals for Realism Enhancement}
\label{exp:alt_reward}
Here, we study whether using an imitation-based reward objective from the pre-training stage is effective for realism alignment.
We use the minimum Average Distance Error (minADE) as a measure of how closely the simulated rollouts align with the ground-truth trajectories.
Another imitation-based reward objective is the cross-entropy between the predicted trajectories and the ground-truth trajectories in token space. 
However, since tokenization introduces error, we opt to use the minADE metric with the RLOO formulation as the reward signal for this experiment.
The results are given in \tabref{tab:reward_ablation}; directly optimizing the meta-metric via $\text{RMM}^{\text{MLOO}}$ is more effective for realism alignment than the imitation-based minADE reward.
We also conduct a post-training experiment with meta-metric as the signal and the RLOO formulation.
Both the MLOO and RLOO reward signals are effective in realism enhancement, with MLOO slightly outperforming RLOO.
We also experiment with using heuristic reward functions using two combinations of collision rate, off-road rate, and minADE values following \cite{gulino2024waymax}.
Although these functions are effective in optimizing their own objectives, they do not lead to an effective improvement in realism at the near-optimal stage.
A detailed comparison including per-scenario paired t-tests is provided in \appref{app:heuristic_rewards}, confirming that $\text{RMM}^{\text{MLOO}}$ significantly outperforms the other reward formulations.

\paragraph{Empirical Reward Variance Scaling Analysis.} We empirically validate our theoretical findings on the variance scaling properties of MLOO and RLOO reward signals by analyzing the reward variance as a function of the number of rollouts ($N$). Using the pre-trained SMART-tiny model, we compute reward statistics for a varying number of rollouts. The results, presented in \figref{fig:var_scaling}, show that MLOO not only yields significantly lower reward variance values but also exhibits a variance that scales according to $1/N^2$. To visually confirm this trend, we fit a $f(N)=\alpha/N^2$ curve to the MLOO variance data. In contrast, the RLOO reward's variance does not scale effectively with the number of rollouts, plateauing for larger values of $N$. This confirms that MLOO provides a more stable and scalable reward signal, which is crucial for efficient fine-tuning.

\begin{table}[!tp]
    \centering
    
      \caption{Effect of goal representation and reward criterion on controllability and realism, evaluated on the full WOMD validation split. \textbf{Bold} and \underline{underline} indicate the best and second best values, respectively. Standard errors are shown in parentheses.}
      \label{tab:goal_completion}
    
    \maybeResize[1]{

      \begin{tabular}{lccccc}
        \toprule
        (Goal rep., Goal criterion) &  Passing Miss Rate $\downarrow$ & Kinematic $\uparrow$  & Interactive $\uparrow$ & Map-Based $\uparrow$ & RMM $\uparrow$ \\
        \midrule
        Goal-Free (RLFTSim) & 16.631 \scriptsize{(9.8e-2)} & \textbf{0.4924} \scriptsize{(5.0e-4)} & \textbf{0.8070} \scriptsize{(4.1e-4)} & \textbf{0.9182} \scriptsize{(5.7e-4)} &  \textbf{0.7830} \scriptsize{(3.3e-4)} \\
        \cmidrule(lr){1-6}
        (Concatenation, Soft) &  \underline{10.473} \scriptsize{(6.5e-2)} & 0.4794  \scriptsize{(4.9e-4)} & 0.8045 \scriptsize{(4.1e-4)} & 0.9134 \scriptsize{(6.0e-4)} & 0.7776 \scriptsize{(3.3e-4)}  \\
        (Concatenation, Hard) & 14.978  \scriptsize{(9.1e-2)} & 0.4791 \scriptsize{(4.9e-4)} & 0.8045 \scriptsize{(4.1e-4)} & 0.9129 \scriptsize{(6.1e-4)} & 0.7774 \scriptsize{(3.3e-4)} \\
        (Indication, Soft) & \;\textbf{9.180}  \scriptsize{(5.9e-2)} & 0.4887 \scriptsize{(4.9e-4)} & \underline{0.8068} \scriptsize{(4.1e-4)} & 0.9175 \scriptsize{(5.7e-4)} & 0.7819 \scriptsize{(3.2e-4)} \\
        (Indication, Hard) & 13.393  \scriptsize{(8.2e-2)} & \underline{0.4916} \scriptsize{(5.1e-4)} & \underline{0.8068} \scriptsize{(4.2e-4)} & \underline{0.9179} \scriptsize{(5.8e-4)} & \underline{0.7827} \scriptsize{(3.4e-4)} \\
        \bottomrule
      \end{tabular}
    }
\end{table}
\begin{figure}[!tp]
\centering
  \includegraphics[trim={0 0 0 1.2cm}, clip, width=0.95\linewidth]{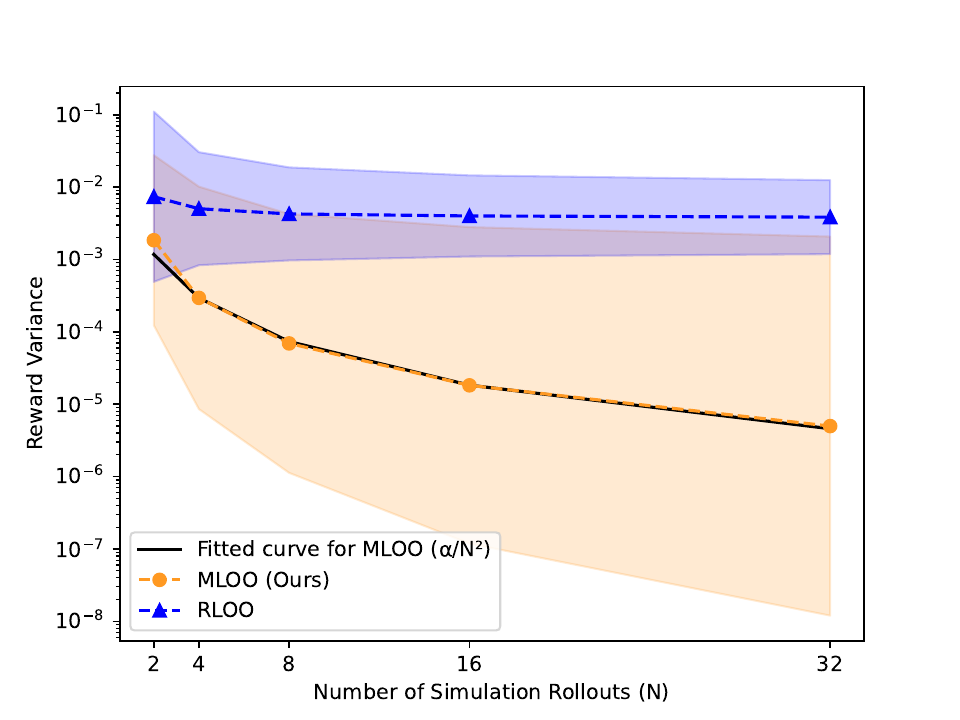}
  \caption{Empirical reward variance of MLOO and RLOO on the validation set, computed over rollouts per scenario for varying $N$. Shaded regions represent $\pm 1$ std.\ in log space.}
\label{fig:var_scaling}

\end{figure}

\begin{figure*}[!tph]
    \centering
    \includegraphics[width=\linewidth]{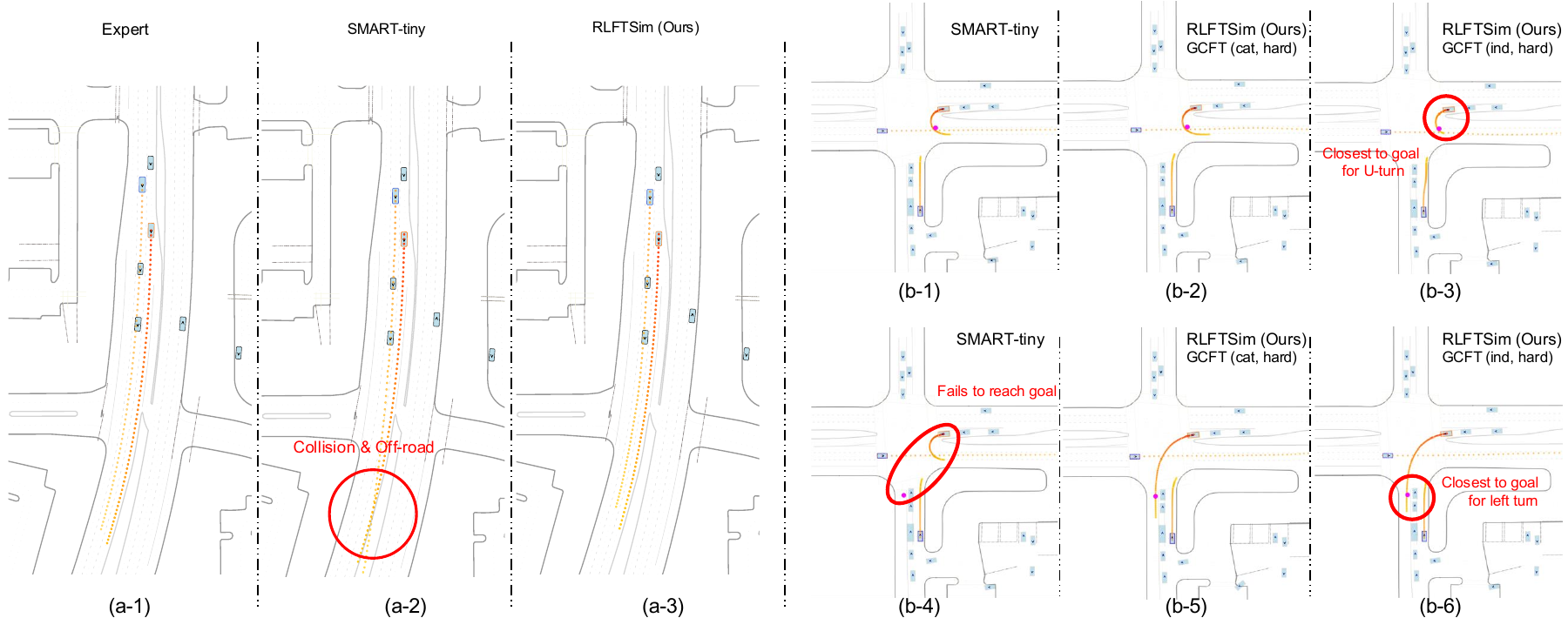}
    \caption{\textbf{Qualitative Evidence of RLFTSim Effectiveness.} 
    \textbf{(a) Realism Enhancement:} Comparison of baseline SMART-tiny (a-2) vs. RLFTSim (a-3) on a challenging intersection scenario. The baseline model generates unrealistic off-road behavior (red trajectory) and a collision with cross-traffic, while RLFTSim produces realistic lane-following behavior that respects traffic rules. 
    \textbf{(b) Controllability Distillation:} Two sets of realistic simulation rollouts of the fine-tuned model are shown for a fixed seed scenario, conditioned on different goal points (magenta points): a U-turn goal (top row) and a left-turn goal (bottom row). In GCFT rollouts, the ego vehicle, depicted with an orange border, achieves the goal point while maintaining realistic interactions with other agents.}
    \label{fig:qualitative1}
\end{figure*}

\subsection{Goal-Conditioned Controllability}
\label{exp:goal_completion}
We train four GCFT variants on the realism-aligned RLFTSim model (\secref{exp:sim_realism}), crossing two goal representations, \textit{concatenation} and \textit{indication} (\secref{sec:controllability}), with two goal-reaching criteria, soft and hard.
Ground-truth final positions serve as goals evaluated via the WOMD miss rate definition~\cite{WOMD_MF}.

As shown in \tabref{tab:goal_completion}, all GCFT methods improve goal passing over the baseline. Indication-based observation maintains higher realism than concatenation, and the soft goal reward yields the best passing rates across both observation approaches.

In \appref{app:controllability},
we further analyze the controllability of the GCFT models by proposing kinematic perturbation and alternative maneuver benchmarks, 
where we evaluate the ability of models to reach goals different from the ground-truth end states. 
In these benchmarks, our GCFT models outperform the baseline across different driving behaviors.

\subsection{Qualitative Analysis}

\figref{fig:qualitative1} demonstrates the effectiveness of both RLFTSim components. Panel (a) shows a representative failure case where the baseline SMART-tiny (a-2) generates unrealistic off-road behavior and a collision with cross-traffic, while RLFTSim (a-3) produces realistic lane-following that respects traffic rules. Panel (b) illustrates controllability distillation: identical initial conditions produce distinct realistic behaviors when conditioned on different goals (U-turn vs. left-turn), validating that GCFT enables goal-directed simulation without compromising realism. These examples illustrate our quantitative improvements in \tabref{tab:waymo_results} and \tabref{tab:goal_completion}. 
Additional qualitative examples are available in \appref{app:qualitative_samples}, and the corresponding videos are available on the project page.

\section{Conclusion}
\label{conclusion}

We present RLFTSim, an RL fine-tuning framework that enhances realism in traffic simulation and enables goal-conditioned scenario generation.
Our approach leverages the realism meta-metric as a reward signal and introduces Meta-metric Leave-One-Out to provide low-variance, dense rewards for sample-efficient training.
Evaluated on the Waymo Open Motion Dataset, RLFTSim achieves state-of-the-art performance across realism metrics while consuming significantly fewer samples due to MLOO's dense reward signal compared to supervised and closed-loop fine-tuning baselines.
We employ Hindsight Experience Replay for goal-conditioned controllability distillation, enabling flexible scenario generation while maintaining realism.

\paragraph{Limitations and Future Work.} The use of token-based representation, which chunks trajectories into half-second segments, potentially reduces responsiveness in highly dynamic traffic scenarios. Goal-conditioned fine-tuning achieves reasonable but imperfect controllability, with room for improvement in goal-reaching rates.
Moreover, RMM is a proxy for realism, and its saturation may partly reflect inadequacies in the metric itself rather than true convergence of simulator quality; developing better realism metrics would further benefit RL-based alignment. As a direction for future work, the MLOO formulation is applicable to other population-based metrics beyond RMM.

{
    \small
    \bibliographystyle{ieeenat_fullname}
    \bibliography{references}
}

\ifdefined\ARXIVVERSION
\maketitlesupplementary
\appendix

\renewcommand{\theequation}{S\arabic{equation}}
\renewcommand{\thefigure}{S\arabic{figure}}
\renewcommand{\thetable}{S\arabic{table}}
\renewcommand{\thetheorem}{S\arabic{theorem}}
\setcounter{equation}{0}
\setcounter{figure}{0}
\setcounter{table}{0}
\setcounter{theorem}{0}

\section{Methodological Details}

\subsection{Detailed RMM Formulation}
\label{app:rmm_derivation}

We provide the complete derivation of the WOSAC's Realism Meta-Metric (RMM) \cite{Montali2023neurips_wosac} that was summarized in \secref{background}.

To compute the WOSAC meta-metric, let $\{\tau_i\}_{i=1}^{N}$ be $N=32$ simulator rollouts sharing the same history and map context, each of length $T$ time steps, and let $\tau^*$ denote the corresponding ground-truth trajectory. 
For each rollout $\tau_i$ and each timestep $t\in\{1,\dots,T\}$, we extract a $D$-dimensional feature vector:

\begin{equation}
\mathbf{f}^{(i)}_t = \bigl(f^{(i)}_{1,t},\,f^{(i)}_{2,t},\,\dots,\,f^{(i)}_{D,t}\bigr),
\end{equation}
whose components include:
\begin{itemize}
    \item \textbf{Kinematic features:} linear/angular speed and acceleration
    \item \textbf{Interactive features:} closest distance to other agents, time-to-collision (TTC), and accident indication  
    \item \textbf{Map-based features:} distance to road boundary, off-road indication, and traffic light violation
\end{itemize}

We compute the same per-timestep features $\mathbf{f}^*_t$ for the ground-truth trajectory.
Each feature dimension $d$ is discretized into bins $\{\mathcal{B}_{d,a,k}\}_{k=1}^K$, where $a$ is the agent index, and $K=20$ is the number of bins.

Given the seed scenario and its group of simulated rollouts, we first form \textbf{time-dependent} empirical distributions:
\begin{equation}
\hat P_{d,a,t}(k) = \frac{1}{N} \sum_{i=1}^N \mathbf{1}\bigl\{f^{(i)}_{d,a,t}\in \mathcal{B}_{d,a,k}\bigr\},
\end{equation}
where $\mathbf{1}\{\cdot\}$ is the indicator function.

Then, we marginalize time to obtain a \textbf{time-independent} histogram:
\begin{equation}
\hat P_{d,a}(k)=\frac{1}{T}\sum_{t=1}^T \hat P_{d,a,t}(k).
\end{equation}

Finally, letting $f^*_{d, a,t}$ fall into bin $k^*_{d,a,t}$ when observed on the ground truth, the WOSAC realism meta-metric is defined as a weighted sum of per-dimension geometric means:
\begin{equation}
\mathrm{RMM} = \sum_{d=1}^D w_d \left[\prod_{(a,t_a) \in V} \hat P_{d,a}(k^*_{d,a,t_a}) \right]^{\frac{1}{|V|}},
\end{equation}
where each weight $w_d\ge0$ reflects the relative importance of feature $d$, and $V = \{(a,t_a); a \in \text{eval. agents}, t_a \in \text{valid time steps} \}$. 
Larger values of RMM indicate that the simulator's distribution of kinematic, interactive, and map-based features more closely aligns with real-world behavior.

\subsection{Proofs and Mathematical Derivations}
\label{app:proofs}

\allowdisplaybreaks 

\begin{proof}[Proof of Proposition~\ref{prop:mloo_pg}]
Starting from the definition of $g$,
\begin{align}
\begin{aligned}
\mathbb{E}[g]
&=
\sum_{i=1}^{N}
\mathbb{E}\Bigg[
\nabla_\theta \log \pi_\theta(\tau_i) \\
&\qquad \times
\left(
\frac{1}{N}\sum_{j=1}^{N}\mathrm{RMM}_{-j}-\mathrm{RMM}_{-i}
\right)
\Bigg].
\end{aligned}
\end{align}
Expanding the two terms gives
\begin{align}
\begin{aligned}
\mathbb{E}[g]
&=
\frac{1}{N}\sum_{j=1}^{N}
\mathbb{E}\!\left[
\left(\sum_{i=1}^{N}\nabla_\theta \log \pi_\theta(\tau_i)\right)\mathrm{RMM}_{-j}
\right] \\
&\quad -
\sum_{i=1}^{N}
\mathbb{E}\!\left[
\nabla_\theta \log \pi_\theta(\tau_i)\,\mathrm{RMM}_{-i}
\right].
\end{aligned}
\end{align}

\proofstep{Step 1: The leave-one-out subtraction term has zero expectation.}
Fix any $i\in\{1,\dots,N\}$. Since $\mathrm{RMM}_{-i}$ is computed from $\tau_{-i}$, it
depends only on the rollout set excluding $\tau_i$. Because the rollouts are
sampled i.i.d., $\tau_i$ is independent of $\tau_{-i}$. Therefore, by the tower
property,
\begin{align}
\begin{aligned}
\mathbb{E}\!\left[ \nabla_\theta \log \right. & \left. \! \pi_\theta(\tau_i)\,\mathrm{RMM}_{-i} \right] = {} \\
& \mathbb{E}_{\tau_{-i}}\!\left[
\mathrm{RMM}_{-i}\,
\mathbb{E}_{\tau_i}\!\left[\nabla_\theta \log \pi_\theta(\tau_i)\right]
\right].
\end{aligned}
\end{align}
Using the score-function identity
\begin{align}
\begin{aligned}
\mathbb{E}_{\tau_i\sim \pi_\theta}\!\left[\nabla_\theta \log \pi_\theta(\tau_i)\right]
&=
\int \pi_\theta(\tau_i)\nabla_\theta \log \pi_\theta(\tau_i)\,d\tau_i \\
&=
\nabla_\theta \int \pi_\theta(\tau_i)\,d\tau_i =0,
\end{aligned}
\end{align}
we obtain
\begin{align}
\mathbb{E}\!\left[
\nabla_\theta \log \pi_\theta(\tau_i)\,\mathrm{RMM}_{-i}
\right]
=0.
\end{align}
Since this holds for every $i$,
\begin{align}
\sum_{i=1}^{N}
\mathbb{E}\!\left[
\nabla_\theta \log \pi_\theta(\tau_i)\,\mathrm{RMM}_{-i}
\right]
=0.
\end{align}
Hence,
\begin{align}
\begin{aligned}
\mathbb{E}[g]
&=
\frac{1}{N}\sum_{j=1}^{N}
\mathbb{E}\!\left[
\left(\sum_{i=1}^{N}\nabla_\theta \log \pi_\theta(\tau_i)\right)\mathrm{RMM}_{-j}
\right].
\end{aligned}
\end{align}

\proofstep{Step 2: Apply the score-function identity to the joint distribution.}
Because the joint rollout distribution factorizes as
$\pi_\theta(\tau_{1:N})=\prod_{i=1}^{N}\pi_\theta(\tau_i)$, we have
\begin{align}
\sum_{i=1}^{N}\nabla_\theta \log \pi_\theta(\tau_i)
=
\nabla_\theta \log \pi_\theta(\tau_{1:N}).
\end{align}
Therefore,
\begin{align}
\begin{aligned}
\mathbb{E}[g]
&=
\frac{1}{N}\sum_{j=1}^{N}
\mathbb{E}\!\left[
\nabla_\theta \log \pi_\theta(\tau_{1:N})\,\mathrm{RMM}_{-j}
\right].
\end{aligned}
\end{align}
Now fix $j$. Writing
$\nabla_\theta \log \pi_\theta(\tau_{1:N})
= \nabla_\theta \log \pi_\theta(\tau_j)
+ \nabla_\theta \log \pi_\theta(\tau_{-j})$,
the $\tau_j$ component contributes zero by the same argument as
Step~1 (independence of $\tau_j$ and $\mathrm{RMM}_{-j}$).
The remaining score $\nabla_\theta \log \pi_\theta(\tau_{-j})$
and $\mathrm{RMM}_{-j}$ both depend only on $\tau_{-j}$, so the
score-function identity on the marginal $\pi_\theta(\tau_{-j})$ gives
\begin{align}
\begin{aligned}
\mathbb{E}\!\left[
\nabla_\theta \log \pi_\theta(\tau_{1:N})\,\mathrm{RMM}_{-j}
\right]
&=
\nabla_\theta \mathbb{E}[\mathrm{RMM}_{-j}].
\end{aligned}
\end{align}
Substituting this into the previous display yields
\begin{align}
\begin{aligned}
\mathbb{E}[g]
&=
\frac{1}{N}\sum_{j=1}^{N}\nabla_\theta \mathbb{E}[\mathrm{RMM}_{-j}] \\
&=
\nabla_\theta \mathbb{E}\!\left[\frac{1}{N}\sum_{j=1}^{N}\mathrm{RMM}_{-j}\right].
\end{aligned}
\end{align}

\proofstep{Step 3: Relate the objective to expected RMM on $N-1$ rollouts.}
By exchangeability, each $\mathrm{RMM}_{-j}$ has the same distribution as the meta-metric
computed on $N-1$ i.i.d. rollouts. Hence,
\begin{align}
\mathbb{E}\!\left[\frac{1}{N}\sum_{j=1}^{N}\mathrm{RMM}_{-j}\right]
=
\mathbb{E}\!\left[\mathrm{RMM}(\tau_{1:N-1})\right].
\end{align}
This proves the claim.
\end{proof}

\begin{proof}[Proof of Proposition~\ref{prop:var-meta-metric-is}]
    The full RMM~\eqref{eq:RMM} is a finite weighted sum $\sum_{d=1}^D w_d\,\mathrm{RMM}_d$ over $D$ feature dimensions. We first establish the $O((N_{\mathrm{eff}}\cdot T)^{-1})$ bound for each per-feature component $\mathrm{RMM}_d = \prod_{k=1}^K \widetilde{P}_k^{\alpha_k}$, written as $\mathrm{RMM}$ for brevity, then lift to the aggregate (Step~6).
    We proceed in six steps.

    \proofstep{Step 1: Direct multinomial sampling estimator.}
    Fix feature $d$ and drop subscripts for clarity. Let $S = NT$ be the total sample size from $N$ rollouts of length $T$. The simulator generates $S$ independent feature observations $\{X_\ell\}_{\ell=1}^S$ directly from the simulator distribution $q$, where each $X_\ell$ falls in bin $k$ with probability $q_k$. The empirical histogram estimator is
    \begin{align}
    \widetilde{P}_k = \frac{1}{S} \sum_{\ell=1}^S \mathbf{1}\{X_\ell = k\}.
    \end{align}
    This estimator directly measures the empirical frequency of simulator features in bin $k$.

    \proofstep{Step 2: Unbiasedness.}
    Since samples are drawn from the simulator distribution $q$, we have $\E_{X_\ell \sim q}[\mathbf{1}\{X_\ell = k\}] = q_k$. By linearity of expectation,
    \begin{align}
    \E_{X_\ell \sim q}\left[\widetilde{P}_k\right] = \frac{1}{S} \sum_{\ell=1}^S \E_{X_\ell \sim q}[\mathbf{1}\{X_\ell = k\}] = q_k.
    \end{align}

    \proofstep{Step 3: Variance and covariance of the estimator.}
    Since samples are independent and drawn from a multinomial distribution with parameters $(S, q_1, \ldots, q_K)$, the variance is:
    \begin{align}
    \Var_{X_\ell \sim q}[\widetilde{P}_k] &= \frac{1}{S} \cdot \Var_{X_\ell \sim q}[\mathbf{1}\{X_\ell = k\}] \\
    &= \frac{1}{S} q_k (1 - q_k) \\
    &= \frac{q_k (1 - q_k)}{S}.
    \end{align}

    For the covariance structure, consider $i \neq j$:
    \begin{align}
    \text{Cov}[\widetilde{P}_i, \widetilde{P}_j] &= \text{Cov}\left[\frac{1}{S} \sum_{\ell=1}^S \mathbf{1}\{X_\ell = i\}, \frac{1}{S} \sum_{\ell=1}^S \mathbf{1}\{X_\ell = j\}\right] \label{eq:cov-def} \\
    &= \frac{1}{S} \text{Cov}[\mathbf{1}\{X = i\}, \mathbf{1}\{X = j\}] \label{eq:cov-linearity} \\
    &= \frac{1}{S} \E[\mathbf{1}\{X = i\} \mathbf{1}\{X = j\}] \notag \\
    &\quad - \frac{1}{S}\E[\mathbf{1}\{X = i\}]\E[\mathbf{1}\{X = j\}] \label{eq:cov-expansion} \\
    &= 0 - \frac{1}{S} q_i q_j \label{eq:cov-mutual-exclusion} \\
    &= -\frac{q_i q_j}{S} \label{eq:cov-final}
    \end{align}
    where \eqref{eq:cov-mutual-exclusion} uses $\mathbf{1}\{X = i\} \mathbf{1}\{X = j\} = 0$ for $i \neq j$.

    \proofstep{Step 4: Variance inflation and effective sample size.}
    The variance of each bin estimator depends on how the simulator distribution $q$ matches the ground truth frequencies $\alpha$. Define the variance inflation factor $\kappa = \sum_{k=1}^K \frac{\alpha_k^2}{q_k}$, which measures the mismatch between ground truth frequencies $\alpha$ and simulator distribution $q$. Here we can derive the following bounds on $\kappa$:

    \textbf{Lower bound:} By Cauchy-Schwarz inequality with vectors $\mathbf{u} = \left(\frac{\alpha_1}{\sqrt{q_1}}, \ldots, \frac{\alpha_K}{\sqrt{q_K}}\right)$ and $\mathbf{v} = (\sqrt{q_1}, \ldots, \sqrt{q_K})$:
    \begin{align}
    \left(\sum_{k=1}^K \frac{\alpha_k}{\sqrt{q_k}} \cdot \sqrt{q_k}\right)^2 \leq \left(\sum_{k=1}^K \frac{\alpha_k^2}{q_k}\right)\left(\sum_{k=1}^K q_k\right).
    \end{align}

    Since $\sum_{k=1}^K \alpha_k = \sum_{k=1}^K q_k = 1$, we get:
    \begin{align}
    1^2 \leq \kappa \cdot 1 \quad \Rightarrow \quad \kappa \geq 1. \label{eq:lower-bound-kappa}
    \end{align}

    Equality holds when $\alpha_k = q_k$ for all $k$ (perfect simulator). When the simulator is biased, $\kappa > 1$.

    \textbf{Upper bound:} Since $\alpha_k \leq 1$ and $q_k > 0$ for all $k$ in the support, we have:
    \begin{align}
    \kappa = \sum_{k=1}^K \frac{\alpha_k^2}{q_k} \leq \frac{1}{\min_{k: q_k > 0} q_k}.
    \end{align}

    The effective sample size is defined as $N_{\mathrm{eff}} = \frac{N}{\kappa}$, giving us the bounds:
    \begin{align}
    N \cdot \min_{k: q_k > 0} q_k \leq N_{\mathrm{eff}} \leq N.
    \end{align}

    \proofstep{Step 5: Meta-metric variance.}
  The meta-metric has the form $\mathrm{RMM} = \prod_{k=1}^K \widetilde{P}_k^{\alpha_k}$ where $\alpha_k$ are fixed ground truth frequencies and $\sum_{k=1}^K \alpha_k = 1$. Taking logarithms:
    \begin{align}
    \log(\mathrm{RMM}) = \sum_{k=1}^K \alpha_k \log(\widetilde{P}_k).
    \end{align}

    By the first-order delta method, since $\Var[\log(\widetilde{P}_k)] \approx \frac{\Var[\widetilde{P}_k]}{(\E[\widetilde{P}_k])^2} = \frac{\Var[\widetilde{P}_k]}{q_k^2}$ and $\text{Cov}[\log(\widetilde{P}_i), \log(\widetilde{P}_j)] \approx \frac{\text{Cov}[\widetilde{P}_i, \widetilde{P}_j]}{q_i q_j}$, we have:
    \begin{align}
    \Var&[\log(\mathrm{RMM})] = \Var\left[\sum_{k=1}^K \alpha_k \log(\widetilde{P}_k)\right] \\
    &= \sum_{k=1}^K \alpha_k^2 \Var[\log(\widetilde{P}_k)] \notag \\
    &\quad + \sum_{i \neq j} \alpha_i \alpha_j \text{Cov}[\log(\widetilde{P}_i), \log(\widetilde{P}_j)] \\
    &= \sum_{k=1}^K \alpha_k^2 \frac{\Var[\widetilde{P}_k]}{q_k^2} \notag \\
    &\quad + \sum_{i \neq j} \alpha_i \alpha_j \frac{\text{Cov}[\widetilde{P}_i, \widetilde{P}_j]}{q_i q_j}.
    \end{align}

    Substituting the variance and covariance formulas from Step 3:
    \begin{align}
    \Var&[\log(\mathrm{RMM})] = \sum_{k=1}^K \alpha_k^2 \frac{q_k(1-q_k)/S}{q_k^2} \notag\\
    &\quad + \sum_{i \neq j} \alpha_i \alpha_j \frac{-q_i q_j/S}{q_i q_j} \\
    &= \sum_{k=1}^K \alpha_k^2 \frac{1-q_k}{S q_k} - \sum_{i \neq j} \alpha_i \alpha_j \frac{1}{S} \\
    &= \frac{1}{S}\left[\sum_{k=1}^K \alpha_k^2 \frac{1-q_k}{q_k} - \sum_{i \neq j} \alpha_i \alpha_j\right] \\
    &= \frac{1}{S}\left[\sum_{k=1}^K \alpha_k^2 \frac{1-q_k}{q_k} - \left(\left(\sum_{k=1}^K \alpha_k \right)^2 - \sum_{k=1}^K \alpha_k^2\right)\right] \\
    &= \frac{1}{S}\left[\sum_{k=1}^K \alpha_k^2 \frac{1-q_k}{q_k} - \left(1 - \sum_{k=1}^K \alpha_k^2\right)\right] \\
    &= \frac{1}{S}\left[\sum_{k=1}^K \alpha_k^2 \left(\frac{1-q_k}{q_k} + 1\right) - 1\right] \\
    &= \frac{1}{S}\left[\sum_{k=1}^K \frac{\alpha_k^2}{q_k} - 1\right] \\
    &= \frac{\kappa - 1}{S} = O\left(\frac{\kappa}{NT}\right) = O\left(\frac{1}{N_{\mathrm{eff}} T}\right). \label{eq:var-final}
    \end{align}
    Finally, applying the delta method to $\mathrm{RMM} = \exp(\log(\mathrm{RMM}))$:
    \begin{align}
    \Var[\mathrm{RMM}] &\approx (\mathrm{RMM})^2 \cdot \Var[\log(\mathrm{RMM})] \notag \\
    &= O\left(\frac{1}{N_{\mathrm{eff}} T}\right).
    \end{align}

    \proofstep{Step 6: Lifting to the aggregate.}
    Each per-feature component satisfies $\Var(\mathrm{RMM}_d) = O((N_{\mathrm{eff},d} \cdot T)^{-1})$ where $N_{\mathrm{eff},d} = N/\kappa_d$.
    By the sub-additivity of standard deviation and $\sum_d w_d = 1$,
    \begin{align}
    \Var\!\left(\textstyle\sum_d w_d\,\mathrm{RMM}_d\right)
    &\leq \left(\textstyle\sum_d w_d\right)^{\!2} \max_d \Var(\mathrm{RMM}_d) \notag \\
    &= \max_d \Var(\mathrm{RMM}_d) \notag \\
    &= O\!\left(\frac{1}{\hat{N}_{\mathrm{eff}} T}\right),
    \end{align}
    where $\hat{N}_{\mathrm{eff}} = N / \max_d \kappa_d$.
    The bounds from Step~4 give $1 \leq \kappa_d \leq 1/\min\limits_{\substack{k:\\ q_{k,d}>0}} q_{k,d}$ for each $d$, so $\max_d \kappa_d \leq 1/\min\limits_{\substack{k,d:\\ q_{k,d}>0}} q_{k,d}$ and hence $\hat{N}_{\mathrm{eff}} \in \bigl[N \cdot \min\limits_{\substack{k,d:\\ q_{k,d}>0}} q_{k,d},\; N\bigr]$.
\end{proof}

\begin{proof}[Proof of Proposition~\ref{prop:mloo-rloo-variance}]
We establish the variance bounds for both estimators using approximations suitable for the leave-one-out setting.

\proofstep{Step 1: Variance of MLOO.}
From \propref{prop:var-meta-metric-is}, each $\mathrm{RMM}_{-i}$ has variance $\Var(\mathrm{RMM}_{-i}) = O(1/((N-1)T))$.

The MLOO estimator can be written as:
\begin{align}
&\mathrm{RMM}_i^{\text{MLOO}} = \frac{1}{N} \sum_{j=1}^N \mathrm{RMM}_{-j} - \mathrm{RMM}_{-i} \\
&= \frac{1}{N} \sum_{j \neq i} \mathrm{RMM}_{-j} + \frac{1}{N} \mathrm{RMM}_{-i} - \mathrm{RMM}_{-i} \\
&= \frac{1}{N} \sum_{j \neq i} \mathrm{RMM}_{-j} - \frac{N-1}{N} \mathrm{RMM}_{-i}.
\end{align}

Since the leave-one-out estimates are correlated (they share $N-2$ common rollouts), we need to account for covariances. Let $\sigma^2 = \Var(\mathrm{RMM}_{-i}) = \frac{C_{\text{var}}}{(N-1)T}$ for some constant $C_{\text{var}} > 0$. For $i \neq j$, we approximate the covariance between $\mathrm{RMM}_{-i}$ and $\mathrm{RMM}_{-j}$ as proportional to the fraction of shared rollouts:
\begin{align}
    \text{Cov}(\mathrm{RMM}_{-i}, \mathrm{RMM}_{-j}) &\approx \frac{N-2}{N-1} \cdot \sigma^2 \notag \\
    &= \frac{N-2}{N-1} \cdot \frac{C_{\text{var}}}{(N-1)T}.
\end{align}

Therefore:
\begin{align}
&\Var(\mathrm{RMM}_i^{\text{MLOO}}) \notag \\
&= \Var\left(\frac{1}{N} \sum_{j \neq i} \mathrm{RMM}_{-j} - \frac{N-1}{N} \mathrm{RMM}_{-i}\right) \\
&= \frac{1}{N^2} \Var\left(\sum_{j \neq i} \mathrm{RMM}_{-j}\right) + \frac{(N-1)^2}{N^2} \Var(\mathrm{RMM}_{-i}) \notag \\
&\quad - 2 \cdot \frac{N-1}{N^2} \text{Cov}\left(\mathrm{RMM}_{-i}, \sum_{j \neq i} \mathrm{RMM}_{-j}\right) \\
&= \frac{1}{N^2} \left[(N-1)\sigma^2 + (N-1)(N-2) \cdot \frac{N-2}{N-1} \sigma^2\right] \notag \\
&\quad + \frac{(N-1)^2}{N^2} \sigma^2 - 2 \cdot \frac{N-1}{N^2} \cdot (N-1) \cdot \frac{N-2}{N-1} \sigma^2 \\
&= \frac{1}{N^2} \left[(N-1) + (N-1)(N-2)^2/(N-1)\right] \sigma^2 \notag \\
&\quad + \frac{(N-1)^2}{N^2} \sigma^2 - 2 \cdot \frac{(N-1)(N-2)}{N^2} \sigma^2 \\
&= \frac{1}{N^2} \left[(N-1) + (N-2)^2\right] \sigma^2 \notag \\
&\quad + \frac{(N-1)^2}{N^2} \sigma^2 - 2 \cdot \frac{(N-1)(N-2)}{N^2} \sigma^2 \\
&= \frac{\sigma^2}{N^2} \left[N-1 + (N-2)^2 + (N-1)^2 \right. \notag \\
&\quad \left. - 2(N-1)(N-2)\right] \\
&= \frac{\sigma^2}{N^2} \left[N-1 + ((N-2) - (N-1))^2\right] \\
    &= \frac{\sigma^2}{N^2} \left[N-1 + 1\right] = \frac{N \sigma^2}{N^2} = \frac{\sigma^2}{N} = \frac{C_{\text{var}}}{N(N-1)T}.
\end{align}

\proofstep{Step 2: Variance of RLOO.}
The RLOO estimator is:
\[
\mathrm{RMM}_i^{\text{RLOO}} = \mathrm{RMM}_i - \frac{1}{N-1} \sum_{j \neq i} \mathrm{RMM}_j.
\]

Since $\mathrm{RMM}_i$ evaluates the meta-metric on a single rollout ($N{=}1$, $S{=}T$), \propref{prop:var-meta-metric-is} gives $\tau^2 = \Var(\mathrm{RMM}_i) = \frac{D}{T}$ for some constant $D > 0$. Because rollouts are i.i.d., the $\mathrm{RMM}_i$ are independent, hence:
\begin{align}
\Var&(\mathrm{RMM}_i^{\text{RLOO}}) = \Var\left(\mathrm{RMM}_i - \frac{1}{N-1} \sum_{j \neq i} \mathrm{RMM}_j\right) \\
&= \Var(\mathrm{RMM}_i) + \Var\left(\frac{1}{N-1} \sum_{j \neq i} \mathrm{RMM}_j\right) \label{eq:rloo-var-1} \\
&= \tau^2 + \frac{1}{(N-1)^2} \Var\left(\sum_{j \neq i} \mathrm{RMM}_j\right) \\
&= \tau^2 + \frac{1}{(N-1)^2} \cdot (N-1) \tau^2 \\
&= \tau^2 + \frac{\tau^2}{N-1} \\
&= \frac{N \tau^2}{N-1} \\
&= \frac{N D}{(N-1) T}.
\end{align}

Since $\frac{N}{N-1}$ is bounded (specifically, $1 < \frac{N}{N-1} \leq 2$ for $N \geq 2$), we have:
\begin{align}
\Var(\mathrm{RMM}_i^{\text{RLOO}}) = O\left(\frac{1}{T}\right). \qedhere
\end{align}
\end{proof}

\subsection{Goal Conditioning Architectures}
\label{app:goal_architectures}

Here we provide further implementation details on the goal conditioning methods discussed in \secref{sec:controllability}.

\paragraph{Agent Token Embedding Concatenation.} An intuitive method of including goal information in the observation is to directly include the goal coordinate,  $\textbf{x}^i_g$, or the relative goal position, $\textbf{r}^i_g = (r^i_g, \phi^i_g)$, in the observed state for each agent, $S_t = \{ S^i_{t'} || \textbf{x}^i_g \mid i \leq N_a, t' \leq t \}$, where $||$ denotes concatenation.
However, since the domain of goals is expansive, it can be difficult for a model to learn an appropriate feature vector. This is because goal coordinates are continuous, and the concatenation process further increases the dimensionality of the embedding vector, which increases the fine-tuning iterations required to generalize.

\paragraph{Positional Encoding Indication.} Instead of including goal coordinates in the input of the agent token embedding encoder, an alternative method is to extend the relative positional encoding (RPE) to include a binary categorical embedding that indicates whether the relationship between agent $i$ and road token $j$ is a goal relationship, which occurs when $j = P^i_g$.
Although this also requires introducing new parameters similar to extending the agent token embeddings, since the input domain is binary, arguably it is easier for a model to learn during the fine-tuning stage. 
Furthermore, given that goal indication is binary and tied to individual polylines, an agent can be unconditioned by providing no goal indication for any polyline. This allows the simulation for that agent to solely focus on maintaining realism. Thus, this method enables a hybrid style simulation where some agents can be conditioned on particular goals while others remain unconditioned.

\begin{table}[!tp]
    \caption{Hyperparameter sweep ranges explored for RLFTSim. Final values are highlighted in \textbf{bold}.}
    \label{tab:hyperparams}
    \centering
    \small
    \begin{tabular}{l c}
        \hline
        \textbf{Hyperparameter} & \textbf{Values swept } \\
        \hline
        Learning rate & 1e-6, \textbf{3e-6}, 1e-5, 3e-5 \\
        Adam~$\beta_1$ & \textbf{0.9}, 0.95 \\
        Warmup steps & \textbf{ 100,} 500\\
        Entropy bonus & \textbf{0.0}, 0.01 \\
        Training rollouts per update & 2, \textbf{4}, 6, 8 \\
        Weight decay & 0.0, \textbf{0.01} \\
        Batch size & \textbf{8},16,32 \\
        Gradient clipping & \textbf{1.0} \\
        KL controller: horizon & \textbf{5} \\
        KL controller: min & \textbf{1e-3} \\
        KL controller: max & \textbf{1e+3} \\
        KL controller: target & 0.005,\textbf{0.01}\\
        Sync. ref. model: steps & \textbf{500} \\
        Sync. ref. model: alpha & \textbf{0.005} \\
        GCFT goal reward weight $\lambda$ & \textbf{0.1}, 0.5 \\
        \hline
    \end{tabular}

\end{table}

\section{Implementation Details}
\label{app:implementation_details}

\paragraph{Model Training.} We train SMART-tiny models on the Waymo Open Motion Dataset  for 32 epochs following the implementation and hyperparameters in \cite{Wu2024SMART}. For the base model training, we use standard supervised learning with cross-entropy loss for next-token prediction.
For the RLFT post-training stage, we use the configuration in \tabref{tab:hyperparams}.
We use the adaptive KL controller from \cite{ziegler2020finetuninglanguagemodelshuman} to control the KL divergence between the model's output distribution and the pre-trained model's output distribution, and sustain model training stability; its hyperparameters are set as in \tabref{tab:hyperparams}.
For the GCFT post-training stage, the hyperparameter configuration is kept the same; however, to ensure that realism is maintained while improving controllability, we use a reward weight of $\lambda{=}0.1$ (\eqref{eq:gcft_reward}). More aggressive goal conditioning can be achieved with higher lambda values.
Experiments for both the base model pre-training and fine-tuning are conducted on a server with Intel Xeon Platinum 8180 CPU @2.50GHz, 728GB RAM, and 8x NVIDIA V100 GPUs each with $32$GB GPU memory.

\paragraph{Dataset.} We use the Waymo Open Motion Dataset  for training and evaluation.
WOMD has 486,995/44,097/44,920 scenarios in the training/validation/test splits, respectively. Each scenario contains up to 128 agents, including agents of type vehicle, pedestrian, and cyclist. Each scenario has a length of 9.1 seconds, consisting of 1.1 seconds for the history input length and 8 seconds for the future simulation horizon.

\paragraph{Evaluation Protocol.} 
The results in \tabref{tab:waymo_results} are based on the private test split of the WOSAC leaderboard.
Unless otherwise specified, all ablation studies and analysis are conducted on a randomly selected 20\% subset of the WOMD validation split (8,800 scenarios).
For realism evaluation, we generate $32$ rollouts per scenario following the WOSAC protocol\footnote{More details can be found in \url{https://waymo.com/open/challenges/2025/sim-agents/}}. The agents that only appear in future time steps are excluded from the simulation and evaluation. The evaluation metrics by default are only based on the specified evaluation agent IDs (the ego vehicle and up to 8 agents tagged as tracks\_to\_predict in the WOMD). Although the other agents are not evaluated, they are included in the simulation and indirectly affect the evaluation metrics for the selected agents.

\begin{table}[!tp]
\centering
\caption{Extended Benchmarking. \textbf{Top:} Performance scaling comparison of our RLFTSim vs. CAT-K \cite{Zhang2024ClosedLoop} with the number of fine-tuning epochs. $\dagger$ indicates a weaker reference model with only 1 epoch of pre-training. \textbf{Middle:} Stronger realism enhancement with a weaker reference model. \textbf{Bottom:} Max realism meta-metric for the ground truth trajectories.}
\label{tab:extended_benchmarking}
\resizebox{\linewidth}{!}{
\begin{tabular}{ll>{\columncolor[gray]{0.8}}cccc}
    \toprule
    \rowcolor{white!50}
    Model & Epoch
    & RMM$\uparrow$
    & Kinematic$\uparrow$
    & Interactive$\uparrow$ 
    & Map-based$\uparrow$  \\
    \midrule
    SMART-tiny \cite{Wu2024SMART} (ref. model) & 0.00 & 0.77692 & 0.48329 & 0.80288 & 0.91135 \\ \cdashline{1-6}
    \multirow{4}{*}{SMART-tiny RLFTSim (Ours)} 
    & 0.25 & 0.78137 & 0.49008 & 0.80807 & 0.91348 \\
    & 0.50 & 0.78183 & 0.48953 & 0.80922 & 0.91364 \\ 
    & 1.00 & 0.78166 & 0.49001 & 0.80897 & 0.91321 \\
    & 1.50 & 0.78113 & 0.48926 & 0.80873 & 0.91242 \\
    \cdashline{1-6}
    \multirow{5}{*}{SMART-tiny CAT-K \cite{Zhang2024ClosedLoop}}
    & 0.25 & 0.78093 & 0.49107 & 0.80435 & 0.91644 \\
    & 0.50 & 0.78101 & 0.49093 & 0.80492 & 0.91603 \\
    & 1.00  & 0.78091 & 0.49124 & 0.80499 & 0.91547 \\
    & 2.00 & 0.78086 & 0.49087 & 0.80498 & 0.91556 \\
    & 5.00 & 0.77983 & 0.48991 & 0.80534 & 0.91270 \\ 
    \midrule
    $\dagger$SMART-tiny \cite{Wu2024SMART} (ref. model) & 0.00 & 0.75073 & 0.46441 & 0.76797 & 0.89220 \\ \cdashline{1-6}
    \multirow{3}{*}{$\dagger$SMART-tiny RLFTSim (Ours)} 
    & 0.50 & 0.76368 & 0.46924 & 0.78618 & 0.90301 \\
    & 1.00 & 0.76360 & 0.46851 & 0.79099 & 0.89701 \\
    & 1.50 & 0.76421 & 0.46721 & 0.79045 & 0.90018 \\
    \midrule
    Oracle & NA & 0.82925 & 0.54976 & 0.85227 & 0.95935 \\
    \bottomrule
\end{tabular}
}
\end{table}

\paragraph{GCFT Reward Design.}
For both soft and hard goals, we assign a binary reward to each agent, indicating whether the agent successfully passes (soft target) or reaches (hard target) its designated goal. The final goal-reaching reward for a scenario is computed by averaging this binary signal over the ego agent and all agents labeled as tracks\_to\_predict.

\section{Additional Experimental Results}

\subsection{Extended Realism Benchmarking}
\label{app:sample_efficiency}
In \tabref{tab:extended_benchmarking}, we provide a more detailed analysis of RLFTSim's fine-tuning performance. The results show that RLFTSim achieves a higher peak RMM score (0.7818) compared to our re-trained SMART-tiny CAT-K model (0.7810) using their public implementation on the same reference model. 
While the margin of improvement over the strong baseline may seem modest, it is important to contextualize this within the performance ceiling. The base SMART-tiny model already achieves an RMM of 0.7769, which is approaching the oracle score of 0.8293, defined as the RMM computed when ground-truth trajectories are used as rollouts. As a model's performance nears this upper bound, further gains become increasingly challenging to achieve.
We observe that both RLFTSim and the re-trained CAT-K baseline reach their peak RMM within the first epoch of fine-tuning, after which performance plateaus.

The effectiveness of RLFTSim is more prominently illustrated when applied to a weaker starting model, as a diagnostic experiment. As shown in the middle section of \tabref{tab:extended_benchmarking}, when fine-tuning a less-optimized SMART-tiny model ($\dagger$SMART-tiny), which is only pre-trained for 1 epoch, with a starting RMM of 0.7507, RLFTSim delivers a substantial performance boost, increasing the RMM to 0.7642 (+1.8\%). This demonstrates our method's capability to enhance the realism of the base model, while the margin of improvement is dependent on the starting performance of the base model.

\begin{figure*}[!tph]
  \newcommand{\topcrop}{0pt}
  \centering
  \resizebox{0.7\textwidth}{!}{
    \begin{minipage}{\textwidth}
      \centering
      \begin{subfigure}[b]{0.47\textwidth}
        \centering
        \includegraphics[trim=0 0 0 \topcrop, clip, width=\textwidth]{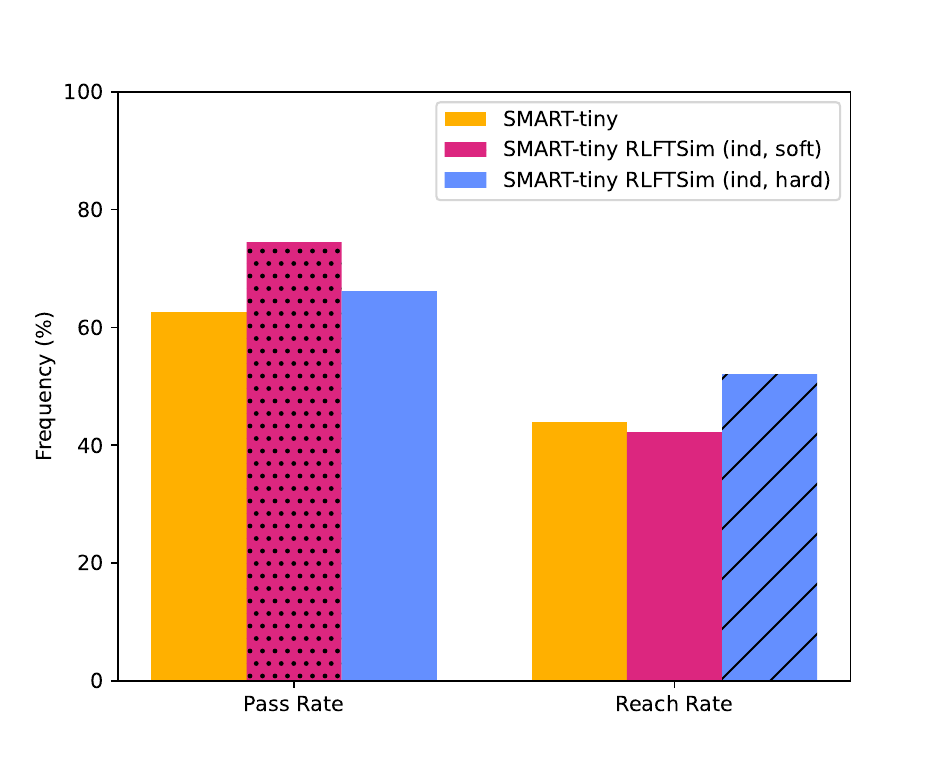}
        \caption{Ground-truth maneuvers}
        \label{fig:modality_control_orig}
      \end{subfigure}
      \hfill
      \begin{subfigure}[b]{0.47\textwidth}
        \centering
        \includegraphics[trim=0 0 0 \topcrop, clip, width=\textwidth]{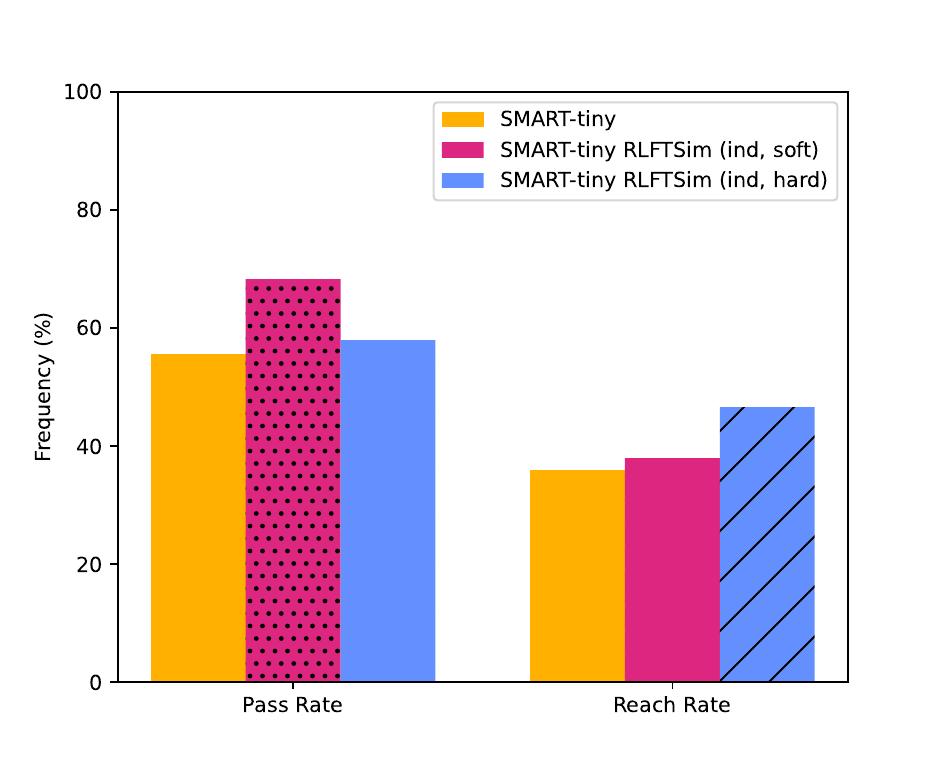}
        \caption{Mixed maneuvers}
        \label{fig:modality_control_some}
      \end{subfigure}

      \begin{subfigure}[b]{0.47\textwidth}
        \centering
        \includegraphics[trim=0 0 0 \topcrop, clip, width=\textwidth]{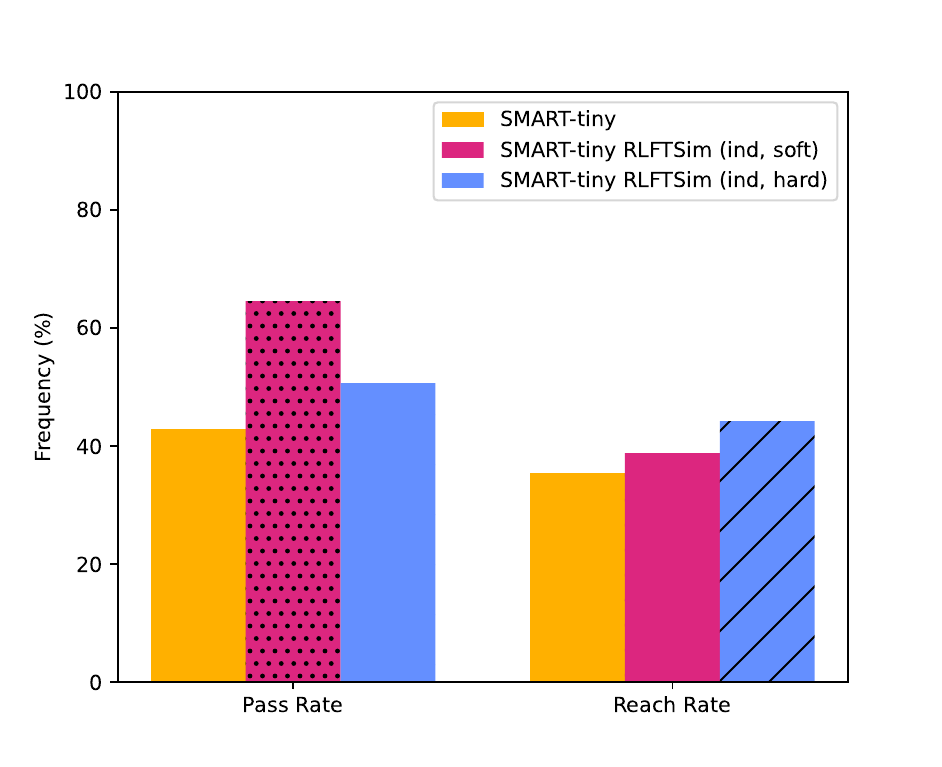}
        \caption{Alternative maneuvers}
        \label{fig:modality_control_no}
      \end{subfigure}
      \hfill
      \begin{subfigure}[b]{0.49\textwidth}
        \centering
        \includegraphics[width=\textwidth]{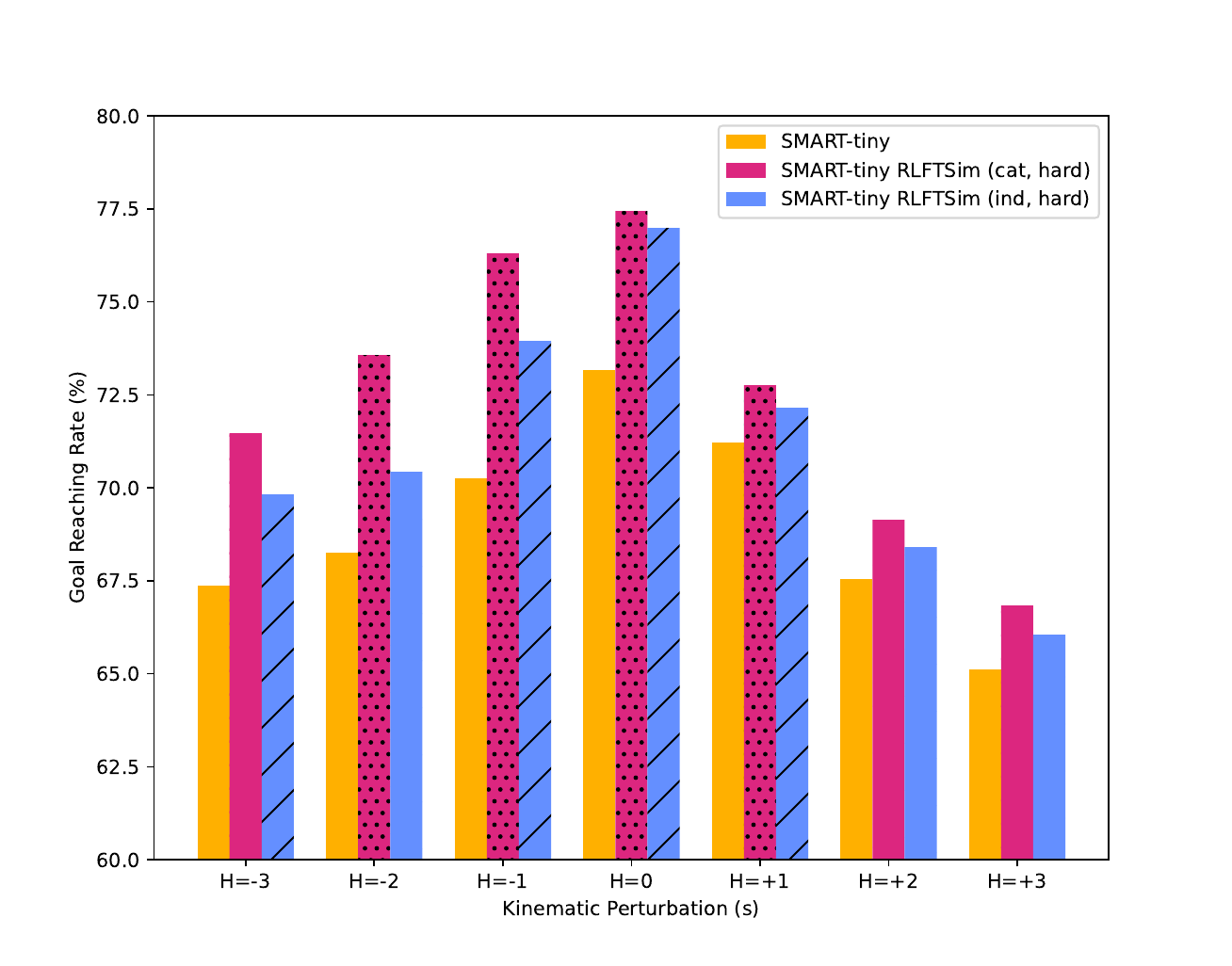}
        \caption{Kinematic controllability}
        \label{fig:kinematic_perturbations}
      \end{subfigure}
    \end{minipage}
  }
  \caption{Controllability benchmark performance across various experimental conditions: (a) all goals are set to ground-truth maneuvers, (b) goals are randomly sampled from all maneuvers, (c) goals are exclusively sampled from alternative maneuvers, and (d) simulation controllability with kinematic perturbations. GCFT models consistently outperform the baseline across all conditions, demonstrating effective controllability distillation.}
  \label{fig:modality_control}
\end{figure*} 
\subsection{Heuristic Rewards}
\label{app:heuristic_rewards}

\begin{table}[!tp]
    \centering
    \caption{Heuristic rewards for the realism meta-metric. All metrics are evaluated on the ego vehicle and agents tagged as \texttt{tracks\_to\_predict} (up to 9 agents). $(\uparrow)$ indicates that larger values are better, and $(\downarrow)$ indicates smaller values are better. Miss rate is computed with the passing goal criterion. \textbf{Bold} and \underline{underline} indicate the best and second best values, respectively.}
    \label{tab:heuristic_rewards}
    \maybeResize[1]{
    \begin{tabular}{l>{\columncolor[gray]{0.9}}ccccc}
    \toprule
    \rowcolor{white!50}
    Reward & RMM $\uparrow$ & Collision (\%) $\downarrow$ & Offroad (\%) $\downarrow$ & ADE (m) $\downarrow$ & minADE (m) $\downarrow$ \\
    \midrule
    SMART-tiny \cite{Wu2024SMART} (ref. model) & 0.7769 & 5.67 & 15.14 & 2.59 & \textbf{1.30} \\
    $\text{RMM}^{\text{MLOO}}$ & \textbf{0.7818} & \underline{4.53} & \underline{14.71} & \underline{2.55} & \underline{1.31} \\
    Collision-offroad-ADE & \underline{0.7788} & 4.93 & 14.73 & \textbf{2.39} & 1.32 \\
    Collision-offroad & 0.7769 & \textbf{4.51} & \textbf{13.95} & 2.62 & 1.36 \\
    \bottomrule
    \end{tabular}
    }
\end{table}

\tabref{tab:heuristic_rewards} presents a comprehensive comparison of different reward formulations for RL-based traffic simulation alignment. 
The results demonstrate the effectiveness of our proposed $\text{RMM}^{\text{MLOO}}$ reward signal compared to heuristic alternatives.
While the collision-offroad reward achieves the lowest collision (4.51\%) and offroad (13.95\%) rates, it sacrifices overall realism, as evidenced by its lower RMM score (0.7769), which is tied with the base model.
In contrast, $\text{RMM}^{\text{MLOO}}$ achieves the best RMM (0.7818), demonstrating superior alignment with realistic driving behaviors while maintaining competitive safety metrics.
The collision-offroad-ade reward, which combines safety metrics with trajectory accuracy, achieves the best ADE (2.39m), but still underperforms $\text{RMM}^{\text{MLOO}}$ in overall realism.
Notably, the base model achieves the best minADE (1.30m), suggesting that pre-trained imitation learning models excel at trajectory accuracy but may not fully capture the distribution of realistic behaviors measured by RMM.
These results validate our design choice of using MLOO as the primary reward signal, as it optimizes the official benchmark metric (RMM) while maintaining reasonable performance across all other metrics, including collision rates, offroad rates, and trajectory errors.

\begin{table}[!tp]
\centering
\caption{Paired t-test on per-scenario RMM scores from the full validation set (\tabref{tab:reward_ablation}) at significance threshold $\alpha=10^{-3}$. \colorbox{green!15}{$\boldsymbol{+}$}\,/\,\colorbox{red!15}{$\boldsymbol{-}$} indicates row is significantly better/worse than column; $\sim$ indicates no significant difference.}
\label{tab:paired_ttest}
\maybeResize[1]{
\small
\begin{tabular}{lcccccc}
\toprule
 & \rotatebox{55}{$\text{RMM}^{\text{MLOO}}$}
 & \rotatebox{55}{$\text{RMM}^{\text{RLOO}}$}
 & \rotatebox{55}{Col.+Off.+ADE}
 & \rotatebox{55}{Ref.\ model}
 & \rotatebox{55}{$\text{minADE}^{\text{RLOO}}$}
 & \rotatebox{55}{Collision+Offroad} \\
\midrule
$\text{RMM}^{\text{MLOO}}$        & \cellcolor{gray!25} & \cellcolor{green!15}$\boldsymbol{+}$ & \cellcolor{green!15}$\boldsymbol{+}$ & \cellcolor{green!15}$\boldsymbol{+}$ & \cellcolor{green!15}$\boldsymbol{+}$ & \cellcolor{green!15}$\boldsymbol{+}$ \\
$\text{RMM}^{\text{RLOO}}$        & \cellcolor{red!15}$\boldsymbol{-}$ & \cellcolor{gray!25} & \cellcolor{green!15}$\boldsymbol{+}$ & \cellcolor{green!15}$\boldsymbol{+}$ & \cellcolor{green!15}$\boldsymbol{+}$ & \cellcolor{green!15}$\boldsymbol{+}$ \\
Col.+Off.+ADE                      & \cellcolor{red!15}$\boldsymbol{-}$ & \cellcolor{red!15}$\boldsymbol{-}$ & \cellcolor{gray!25} & $\sim$ & \cellcolor{green!15}$\boldsymbol{+}$ & \cellcolor{green!15}$\boldsymbol{+}$ \\
Ref.\ model                       & \cellcolor{red!15}$\boldsymbol{-}$ & \cellcolor{red!15}$\boldsymbol{-}$ & $\sim$ & \cellcolor{gray!25} & \cellcolor{green!15}$\boldsymbol{+}$ & \cellcolor{green!15}$\boldsymbol{+}$ \\
$\text{minADE}^{\text{RLOO}}$     & \cellcolor{red!15}$\boldsymbol{-}$ & \cellcolor{red!15}$\boldsymbol{-}$ & \cellcolor{red!15}$\boldsymbol{-}$ & \cellcolor{red!15}$\boldsymbol{-}$ & \cellcolor{gray!25} & \cellcolor{green!15}$\boldsymbol{+}$ \\
Collision+Offroad                  & \cellcolor{red!15}$\boldsymbol{-}$ & \cellcolor{red!15}$\boldsymbol{-}$ & \cellcolor{red!15}$\boldsymbol{-}$ & \cellcolor{red!15}$\boldsymbol{-}$ & \cellcolor{red!15}$\boldsymbol{-}$ & \cellcolor{gray!25} \\
\bottomrule
\end{tabular}
}
\end{table}
 
While \tabref{tab:reward_ablation} reports standard errors of the mean RMM, these marginal intervals do not account for the paired structure of the evaluation: all methods are assessed on the same scenarios. A two-sided paired t-test on per-scenario RMM differences across $N{=}44{,}097$ validation scenarios is therefore more powerful, as it removes inter-scenario variance. \tabref{tab:paired_ttest} reports the resulting pairwise outcomes. $\text{RMM}^{\text{MLOO}}$ significantly outperforms all other reward formulations.

\subsection{Extended Controllability Analysis}
\label{app:controllability}

Here we provide more analysis on the controllability benchmarking discussed in \secref{exp:goal_completion} and present detailed results.
A key motivation for a controllable simulator is the ability to provide externally supplied behaviors to individual agents, especially those diverging from the agent’s original trajectory. To probe this capability, we introduce two benchmarks that assess how effectively a simulation can be conditioned on novel, goal-directed behaviors.

\paragraph{Kinematic Controllability.} This benchmark tests the ego's ability to reach goal coordinates displaced in time from its ground-truth endpoint by a signed horizon $H \in \{-3,\ldots,+3\}$ seconds (\figref{fig:kinematic_perturbations}). For $H<0$, the goal is the ground-truth position at time $T+H$; for $H>0$, it is obtained by propagating the ground-truth kinematic state at $T$ with a constant-velocity bicycle model for $H$ seconds. $H=0$ is the unperturbed ground-truth goal.

Both hard-reward GCFT variants, (cat, hard) and (ind, hard), outperform the baseline SMART-tiny across the full range and on both sides of $H=0$. Concatenation (cat) leads indication (ind) across tested horizons, consistent with the hard-reward ordering on the maneuver benchmark (\tabref{tab:modality_control}).

\paragraph{Maneuver Controllability.}
We construct a benchmark of 100 scenarios from the WOMD evaluation set in which the ego vehicle has multiple valid maneuvers. For each scenario, we manually select alternative goal coordinates (\secref{sec:controllability}) corresponding to one of six maneuver types: drive straight, left turn, right turn, left U-turn, and lane change left or right. The benchmark contains both ground-truth and alternative maneuvers. Alternative maneuvers are harder than ground-truth ones: the latter come from the pre-training distribution, while the former require generalization beyond it.

\figref{fig:modality_control} shows goal-completion rates for the two GCFT (ind) variants against the baseline SMART-tiny across three goal-set conditions of increasing difficulty: all goals matching the ground-truth maneuver (\figref{fig:modality_control_orig}), a half-half mix of ground-truth and alternative goals (\figref{fig:modality_control_some}), and alternative goals only (\figref{fig:modality_control_no}). GCFT (ind) remains competitive with the baseline when goals match the ground truth, and outperforms it under both the mixed and alternative-only conditions, where the model must simulate maneuvers different from this scenario's recorded ground truth.

\begin{table}[!tp]
    \centering
    \caption{Analysis on maneuver controllability benchmark. Targets are only chosen from the set that does not contain the  ground-truth maneuver. Only the ego vehicle is evaluated. Note that lower absolute goal completion rates are expected as GT maneuvers are excluded from the benchmark.}
    \label{tab:modality_control}
    \maybeResize[0.70]{
      \small
      \begin{tabular}{lcc}
      \toprule
        \multirow{2}{*}{Method} & \multicolumn{2}{c}{Goal Completion Rate}\\
        \cmidrule(lr){2-3}
        & Reach Rate (\%) $\uparrow$ & Pass Rate (\%) $\uparrow$  \\
        \midrule
        SMART-tiny \cite{Wu2024SMART} (ref. model) & 35.37 & 42.89  \\ 
        \cmidrule(lr){1-3}
        RLFTSim (cat, soft) & 37.59 & \textbf{75.34} \\
        RLFTSim (cat, hard) & \textbf{50.00} & \underline{68.37} \\
        RLFTSim (ind, soft) & 38.78 & 64.46 \\
        RLFTSim (ind, hard) & \underline{44.22} & 50.68 \\
        \bottomrule
      \end{tabular}
    }
\end{table}
 
\tabref{tab:modality_control} reports goal-completion rates, based on their definition in \secref{sec:controllability}, on the alternative-only condition after 20K GCFT steps and with a goal reward weight of $\lambda{=}0.1$ (\eqref{eq:gcft_reward}). All four GCFT variants improve upon the baseline SMART-tiny with pass rate gains (8--32 percentage points) larger than reach rate gains (2--15 percentage points). Concatenation (cat) dominates indication (ind) in three of the four (metric, reward) cells; the exception is soft-reward reach rate, where (ind) slightly leads. The hard reward attains the best reach rate (50.0\%, cat, hard) and the soft reward the best pass rate (75.3\%, cat, soft), matching their respective reward definitions (\secref{sec:controllability}).

\begin{table}[!tp]
\centering
\caption{Model agnosticism ablation study. Experiments are done using 20\% of the WOMD validation split.}
\label{tab:model_agnostic}
\resizebox{\linewidth}{!}{
\begin{tabular}{l>{\columncolor[gray]{0.8}}cccc}
    \toprule
    \rowcolor{white!50}
    Model
    & RMM$\uparrow$
    & Kinematic$\uparrow$
    & Interactive$\uparrow$ 
    & Map-based$\uparrow$  \\
    \midrule
    TrafficBots V1.5 \cite{zhang2024trafficbotsv15} & 0.71743 & 0.42712 & 0.73166 & 0.86502 \\
    TrafficBots V1.5 RLFTSim (Ours)
    & 0.72305 & 0.43209 & 0.73773 & 0.87043 \\
    \bottomrule
\end{tabular}
}
\end{table}

\subsection{Model Agnosticism}
\label{app:agnostic_ablations}

To demonstrate the capability of RLFTSim to be model-agnostic, we fine-tune the TrafficBots V1.5 \cite{zhang2024trafficbotsv15} model using RLFTSim to enhance the performance further. \tabref{tab:model_agnostic} also compares the TrafficBots V1.5 model after 1 epoch of pre-training with the TrafficBots V1.5 RLFTSim model after a further fine-tuning for 12{,}000 steps. Pre-training follows the default configuration of~\cite{zhang2024trafficbotsv15}, and for RLFTSim post-training, we reuse the hyperparameters from \tabref{tab:hyperparams}.  The improved RMM from 0.7174 to 0.7231 shows that RLFTSim extends beyond the discrete-token SMART baseline to a continuous-action, VAE-based architecture for TrafficBots V1.5.

\subsection{Extended Qualitative Examples}
\label{app:qualitative_samples}

The collision and off-road examples (\figref{fig:collision1}, \figref{fig:collision2}, \figref{fig:collision3}, \figref{fig:offroad1}) show how RLFTSim addresses safety violations present in the baseline SMART-tiny model. The pre-trained model generates vehicle-pedestrian collisions, rear-end crashes, right-of-way violations, and off-road excursions, while RLFTSim produces behaviors that respect traffic rules and adhere to drivable areas. These improvements correspond to the enhanced interactive and map-based metrics reported in \tabref{tab:waymo_results}.

The goal-conditioned examples (\figref{fig:gcft_redlight}, \figref{fig:gcft_stopsign}, \figref{fig:gcft_parking}) showcase how GCFT distills controllability in the simulation, allowing for specific goals to be specified that the fine-tuned model is capable of reaching.

\newcommand*{\figwidth}{0.80\linewidth}

\begin{figure*}[!tph]
    \centering
    \includegraphics[width=\figwidth]{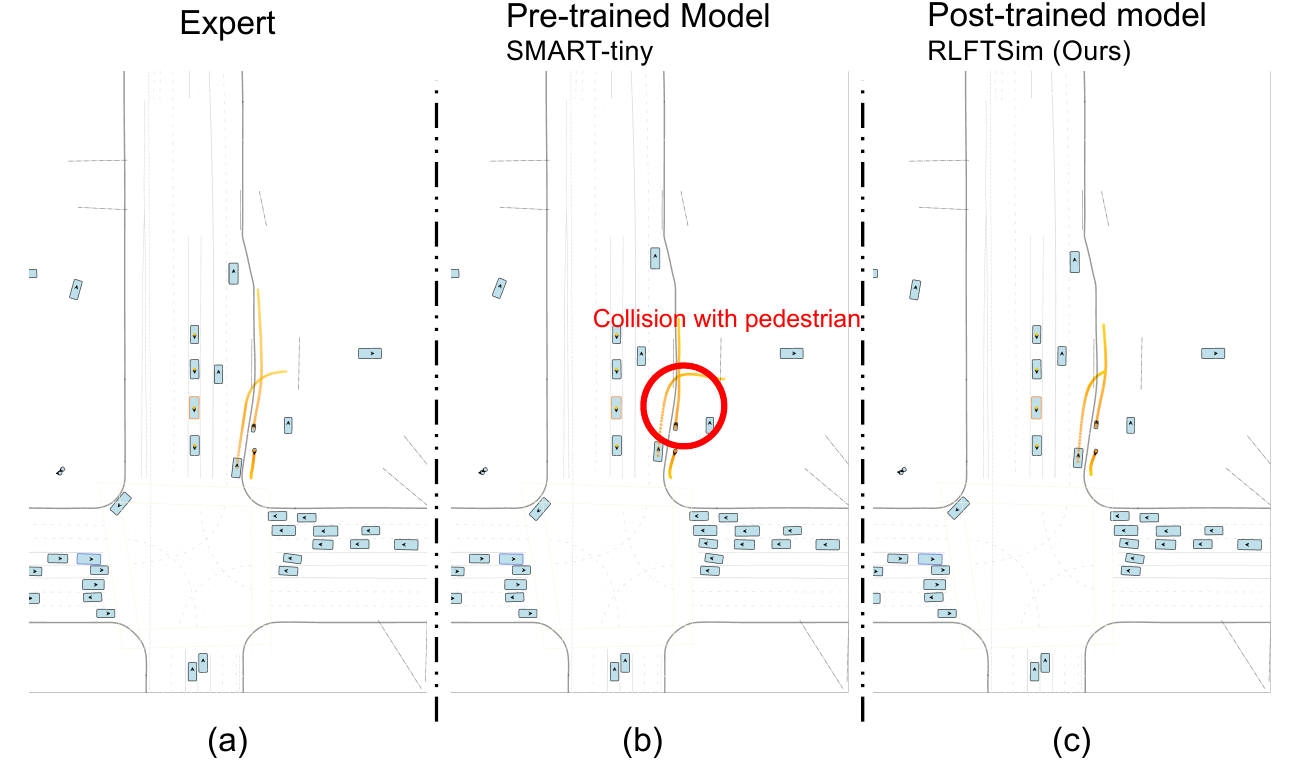}    
    \caption{\textbf{Qualitative Sample - Collision 1.} In the simulation for the pre-trained model (b), the vehicle entering the circle fails to yield to the pedestrian and collides with it. There is no collision in the case of the expert data (a) and RLFTSim (c).
    }
    \label{fig:collision1}
\end{figure*}

\begin{figure*}[!tph]
    \centering
    \includegraphics[width=\figwidth]{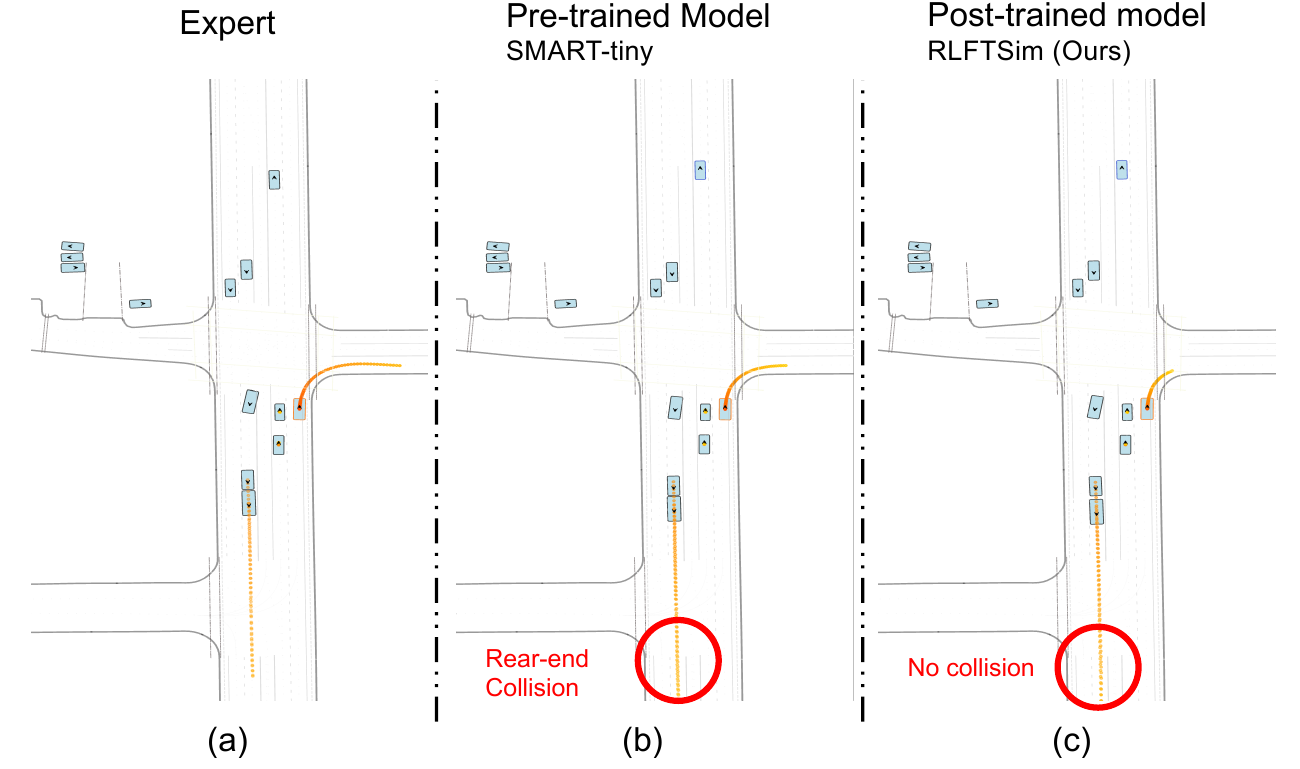}    
    \caption{\textbf{Qualitative Sample - Collision 2.} (b) For the pre-trained model, there is a rear-end collision between two vehicles in the focused zone. (c) However, the post-trained model avoids this accident.
    }
    \label{fig:collision2}
\end{figure*}

\begin{figure*}[!tph]
    \centering
    \includegraphics[width=\figwidth]{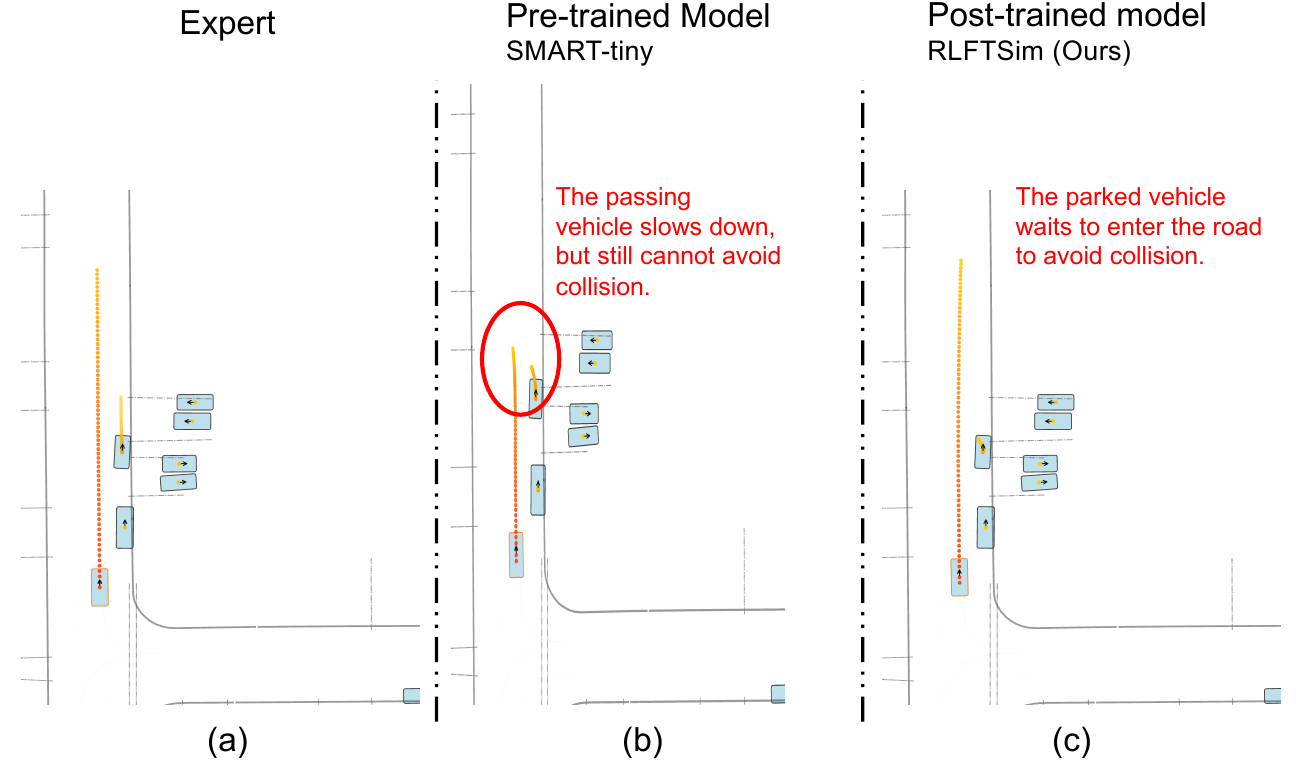}    
    \caption{\textbf{Qualitative Sample - Collision 3.} 
    (a) The parked vehicle starts to move forward, but does not enter the road to avoid a collision. (b) For the pre-trained model, the parked vehicle attempts to enter the road, which leads to a collision with the passing vehicle. The passing vehicle tries to slow down, but it cannot avoid the collision. (c) The parked vehicle waits for the road to clear and then enters the road.
    }
    \label{fig:collision3}
\end{figure*}

\begin{figure*}[!tph]
    \centering
    \includegraphics[width=\figwidth]{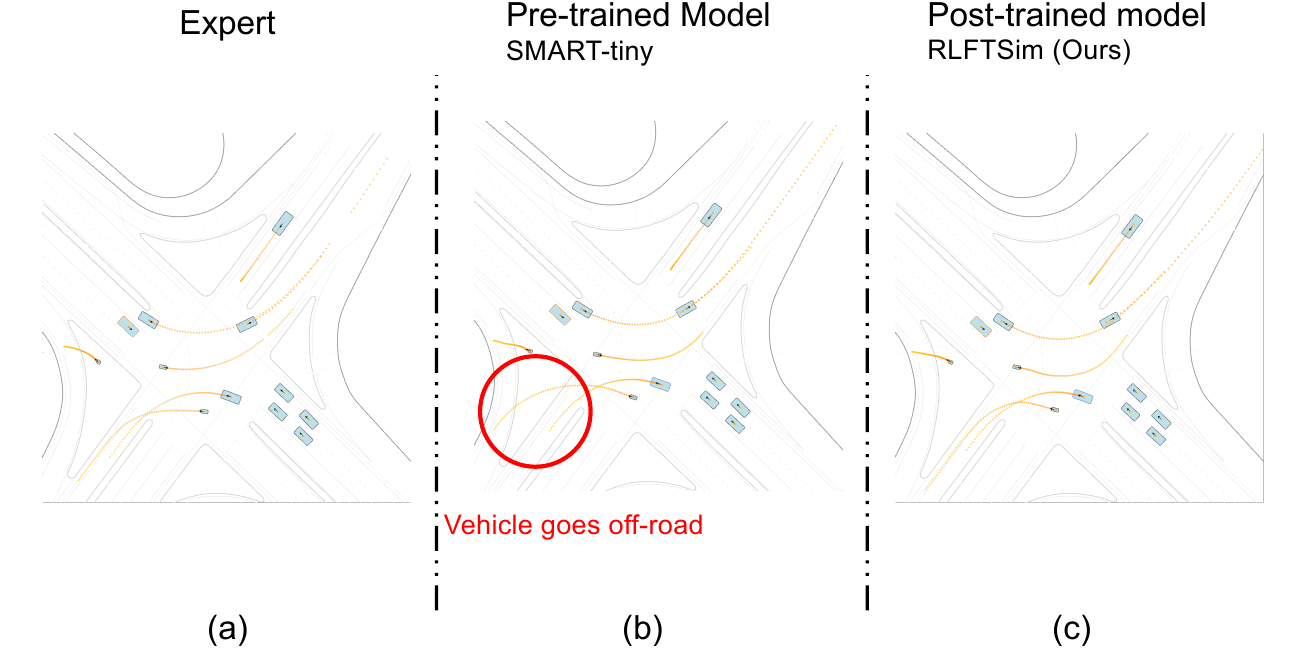}    
    \caption{\textbf{Qualitative Sample - Off-road 1.} (b) The cyclist does not respect the drivable area and goes off-road. For the expert data (a) and the RLFTSim model (c), the cyclist adheres to the drivable area.
    }
    \label{fig:offroad1}
\end{figure*} 
\newcommand*{\figwidthgcft}{0.80\linewidth}

\begin{figure*}[!tp]
    \centering
    \includegraphics[width=\figwidthgcft]{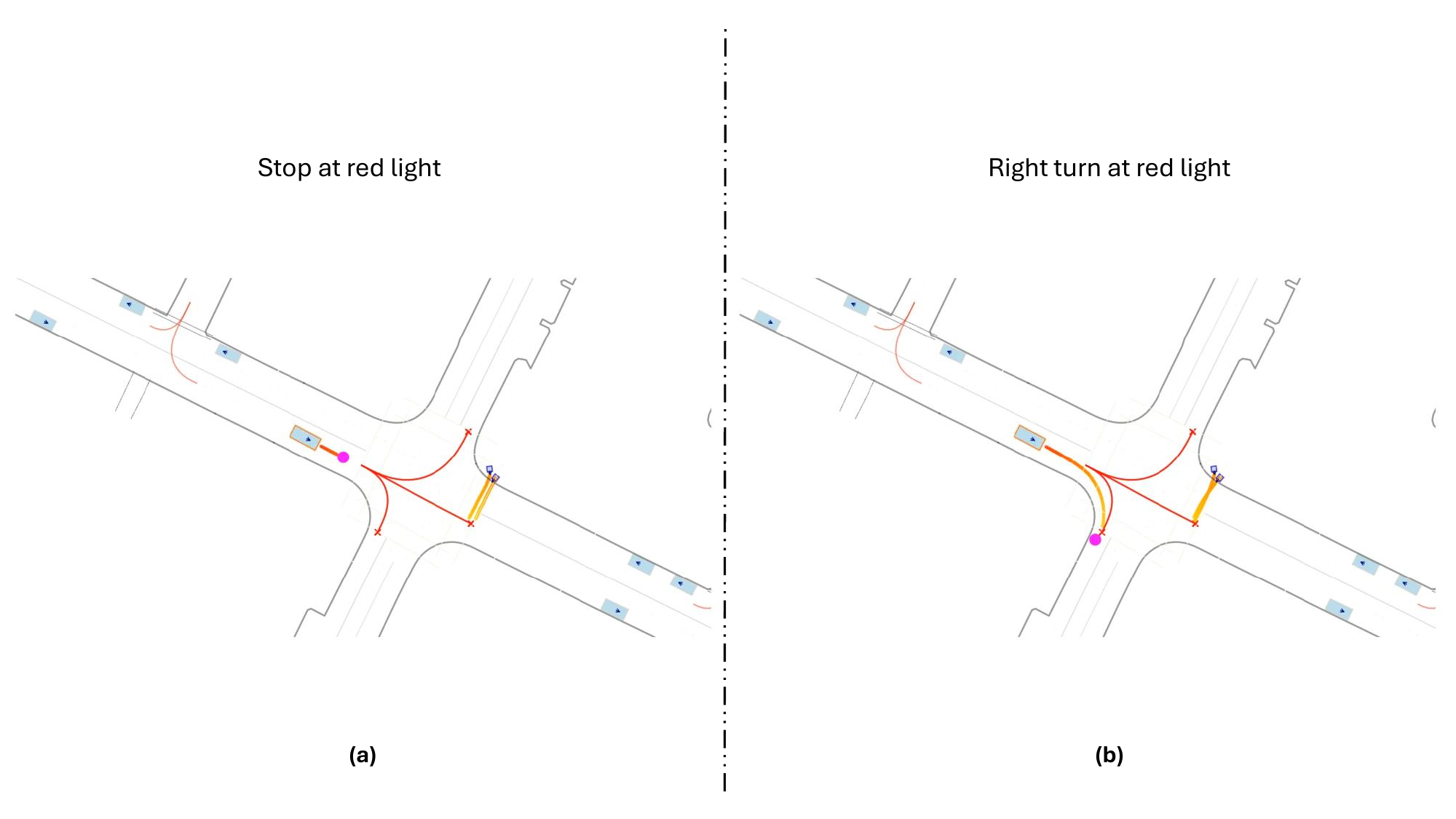}    
    \caption{\textbf{Qualitative Sample - GCFT Red Light} Fine-tuning with GCFT enables greater control over scenario diversity. In this example, the base model only produces rollouts where the ego vehicle stops at the red light. In contrast, GCFT allows for the generation of rollouts where the ego performs either a right turn at the red light (b) or a full stop (a).
    }
    \label{fig:gcft_redlight}
\end{figure*}

\begin{figure*}[!tp]
    \centering
    \includegraphics[width=\figwidthgcft]{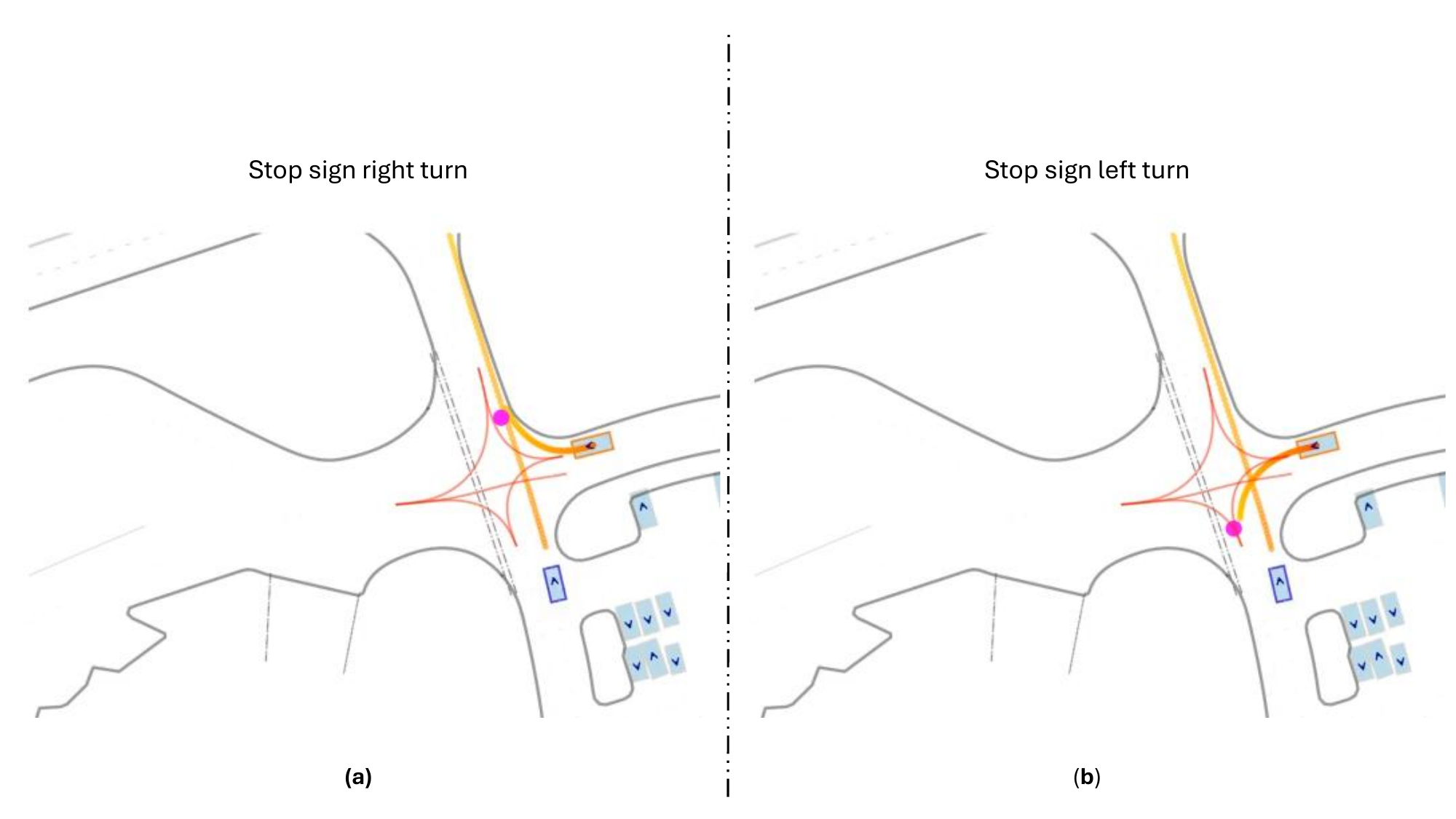}    
    \caption{\textbf{Qualitative Sample - GCFT Stop Sign} Fine-tuning with GCFT enables greater control over scenario diversity. In this example, the base model only produces rollouts where the ego vehicle turns right after stopping. In contrast, GCFT allows for the generation of rollouts where the ego vehicle can turn either right (a) or left (b) at the stop sign.
    }
    \label{fig:gcft_stopsign}
\end{figure*}

\begin{figure*}[!tp]
    \centering
    \includegraphics[width=\figwidthgcft]{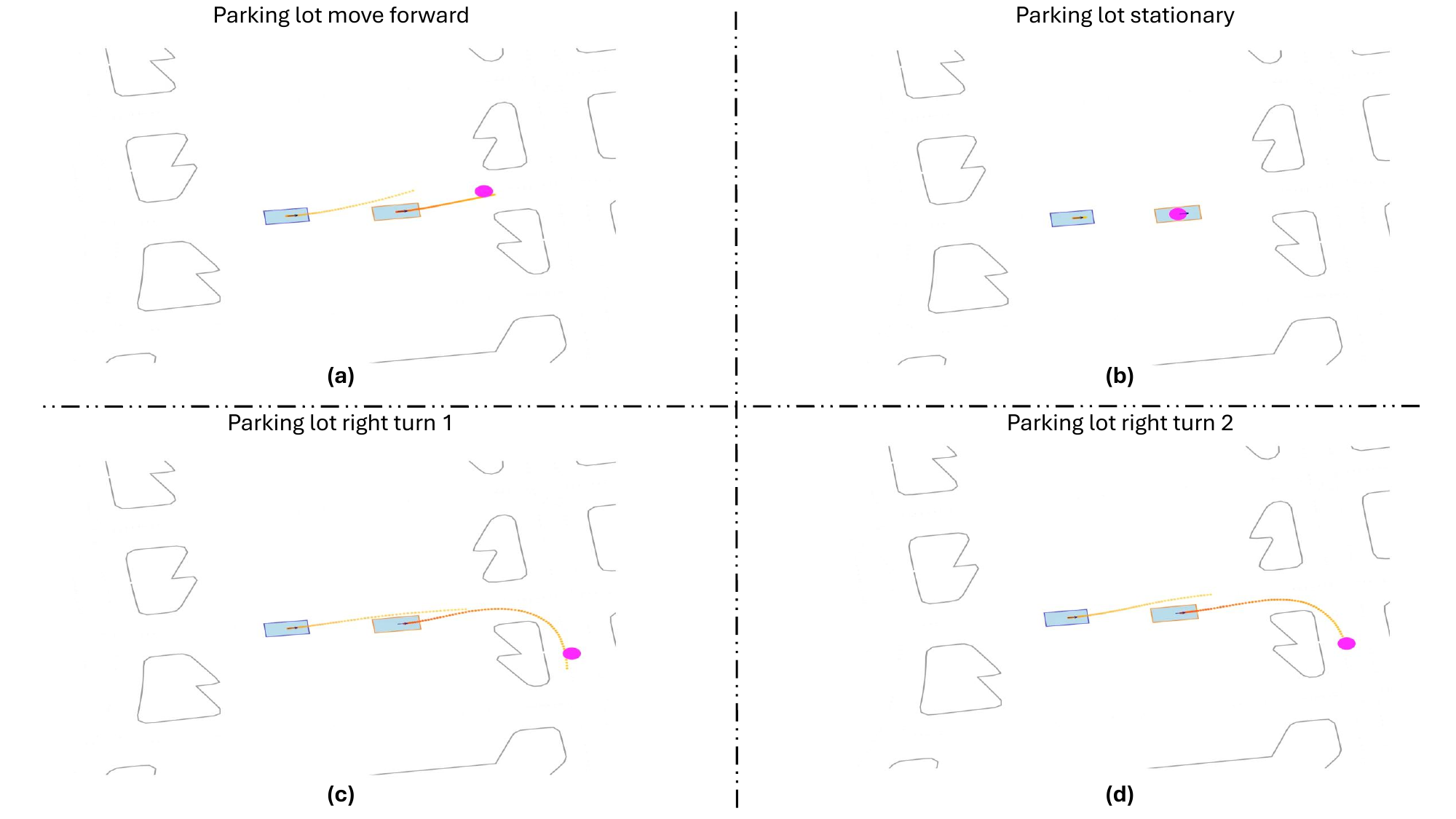}    
    \caption{\textbf{Qualitative Sample - GCFT Parking Lot} Fine-tuning with GCFT allows for behavior creation from otherwise static objects. In the base SMART model, the ego vehicle always remains stationary in this scenario. After GCFT, we can specify for the agent to move forward (a), remain stationary (b), or perform a right turn (c \& d).
    }
    \label{fig:gcft_parking}
\end{figure*}  \fi

\end{document}